\documentclass{article} 

\PassOptionsToPackage{numbers,sort&compress}{natbib}

\usepackage[preprint]{neurips_2026}

\usepackage[utf8]{inputenc}
\usepackage[T1]{fontenc}

\usepackage{microtype}
\usepackage{graphicx}
\usepackage{mathtools}
\usepackage{subcaption}
\usepackage{booktabs}
\usepackage{placeins} 
\graphicspath{{./}{figs/}}
\DeclareGraphicsExtensions{.pdf,.png,.jpg,.jpeg}

\usepackage[hidelinks]{hyperref}

\usepackage{amsmath,amssymb,amsfonts}
\usepackage{amsthm}

\usepackage{xcolor}

\newif\ifshowAYedits
\showAYeditstrue

\newcommand{\E}{\mathbb{E}}
\newcommand{\R}{\mathbb{R}}
\newcommand{\Tr}{\ensuremath{\mathrm{Tr}}}
\newcommand{\Id}{\ensuremath{\mathrm{I}}}

\newcommand{\AIRM}{\ensuremath{d_{\mathrm{AIRM}}}}

\newcommand{\Spp}{\mathbb{S}_{++}}

\newcommand{\sras}{\ensuremath{\mathrm{S\text{-}RAS}}}
\newcommand{\epsreg}{\ensuremath{\varepsilon_{\mathrm{reg}}}}
\newcommand{\dinf}{\ensuremath{d_{\infty}}}

\newcommand{\op}{\ensuremath{\mathrm{op}}}

\theoremstyle{plain}
\newtheorem{theorem}{Theorem}
\newtheorem{proposition}{Proposition}
\newtheorem{lemma}{Lemma}
\theoremstyle{definition}
\newtheorem{definition}{Definition}
\theoremstyle{remark}
\newtheorem*{remark}{Remark}

\title{
Beyond Activation Alignment:
The Geometry of Neural Sensitivity
}

\author{%
  \textbf{Amirhossein Yavari$^1$, Farnaz Zamani Esfahlani$^{1,2}$\thanks{Corresponding author.}} \\
  $^1$ Stephenson School of Biomedical Engineering, University of Oklahoma \\
  $^2$ Data Science and Analytics Institute, University of Oklahoma \\
  Norman, OK, USA \\
\texttt{amirhossein.yavari@ou.edu, fzamani@ou.edu}
}

\begin{document}
\maketitle

\begin{abstract}
Activation-alignment measures such as Representational Similarity Analysis (RSA), Canonical Correlation Analysis (CCA), and Centered Kernel Alignment (CKA) are widely used to compare biological and artificial neural representations. Recent theoretical work interprets many of these methods as assessing agreement between optimal linear readouts over broad families of global tasks. However, agreement at the level of global readouts does not determine how a system uses local stimulus evidence. Specifically, representations may align in activation space yet differ in their sensitivity to small perturbations. To address this challenge, we introduce a complementary framework based on local decodable information, which focuses on a representation's ability, under noise, to discriminate small perturbations within a specified stimulus-coordinate subspace. Building on Fisher information and local representation geometry, we summarize each representation using the expected projected pullback/Fisher metric over that subspace. This formulation induces a second-moment family of local discrimination tasks, for which the resulting operator provides a minimal, complete dataset-level summary of expected discriminability. We compare these regularized signatures using a log-spectral distance on the manifold of symmetric positive definite (SPD) matrices, yielding the Spectral Riemannian Alignment Score (S-RAS) and a uniform multiplicative certificate over the corresponding family of lifted task values. Empirically, this framework enables the recovery of corresponding layers across independently trained artificial neural networks, supports transferable class-conditional probes, reveals controlled dissociations between standard and robust training, and uncovers stimulus-coordinate family effects across mouse visual cortex using the Allen Brain Observatory static gratings dataset.

\end{abstract}

\section{Introduction}
One of the central goals in machine learning and computational neuroscience is to compare internal representations in terms of their shared computations, invariances, and underlying mechanisms \cite{kriegeskorte2008representational,raghu2017svcca,kornblith2019cka,deVries2020AllenBO,Stringer2019HighDim}. A number of widely used methods for representational comparison include Representational Similarity Analysis (RSA), Canonical Correlation Analysis (CCA), and Centered Kernel Alignment (CKA), which are primarily focused on quantifying geometric similarity in stimulus-evoked activity patterns \cite{kriegeskorte2008representational, raghu2017svcca, kornblith2019cka}. While these approaches are effective for measuring representational alignment across models and biological systems, they are limited in determining whether aligned representations rely on the same local stimulus evidence to support downstream computations \cite{harvey2024what,davari2022reliability,Feather2025PrincipalDistortions}.

Several complementary perspectives including decodable-information and local-geometry approaches have been proposed to address this limitation \cite{harvey2024what, Berardino2017EigenDistortions,Zhou2024Discriminability,Lipshutz2024LocalInfoGeometry,Feather2025PrincipalDistortions,cayco_gajic2026msa}. The decoder-based perspective defines representation similarity in terms of downstream accessibility, characterizing representations by which variables are linearly decodable and therefore available to subsequent computational stages. Under this view, similarity measures evaluate whether different representations support comparable performance across a shared family of decoding tasks, thus capturing agreement among optimal linear readouts rather than alignment at the level of individual units \cite{harvey2024what}. Nevertheless, decoder agreement does not uniquely determine the underlying encoding process where distinct representations can support similar linear readouts while relying on different stimulus directions. Consistent with this limitation, prior work \cite{davari2022reliability} shows that activation-based similarity measures (e.g., CKA) are not invariant to simple post hoc transformations of representational point clouds, such as fixed subset translations followed by global mean-centering.

Another complementary perspective emphasizes local discriminability by studying how representations respond to small input changes, particularly how perturbations in stimulus or task coordinates are transformed within the representation. To quantify this, local geometry-based methods use derivatives, Jacobians, and Fisher or pullback metrics to identify directions in stimulus space that are amplified, suppressed, or made discriminable \cite{Green1966SDT,DayanAbbott2001,BrunelNadal1998FI,Pouget2000PopCodes,Berardino2017EigenDistortions,Zhou2024Discriminability,Lipshutz2024LocalInfoGeometry,Feather2025PrincipalDistortions}. 
However, local geometry does not specify the choice of perturbation family, how information should be aggregated across stimuli, or what dataset-level summary is sufficient to characterize the resulting discrimination tasks.

Motivated by these limitations, we ask whether two representations make the same small changes discriminable within a specified stimulus or task-coordinate family. We formalize this through expected local discriminability where perturbations are restricted to a chosen coordinate subspace, summarized by their second moments, and averaged over the stimulus distribution. This yields dataset-level signatures that can be compared across models where the resulting operator fully characterizes expected discriminability for the chosen second-moment perturbation family. The overall framework is shown in Figure~\ref{fig:cartoon}.

\begin{figure*}[t]
  \centering
  \includegraphics[width=\textwidth]{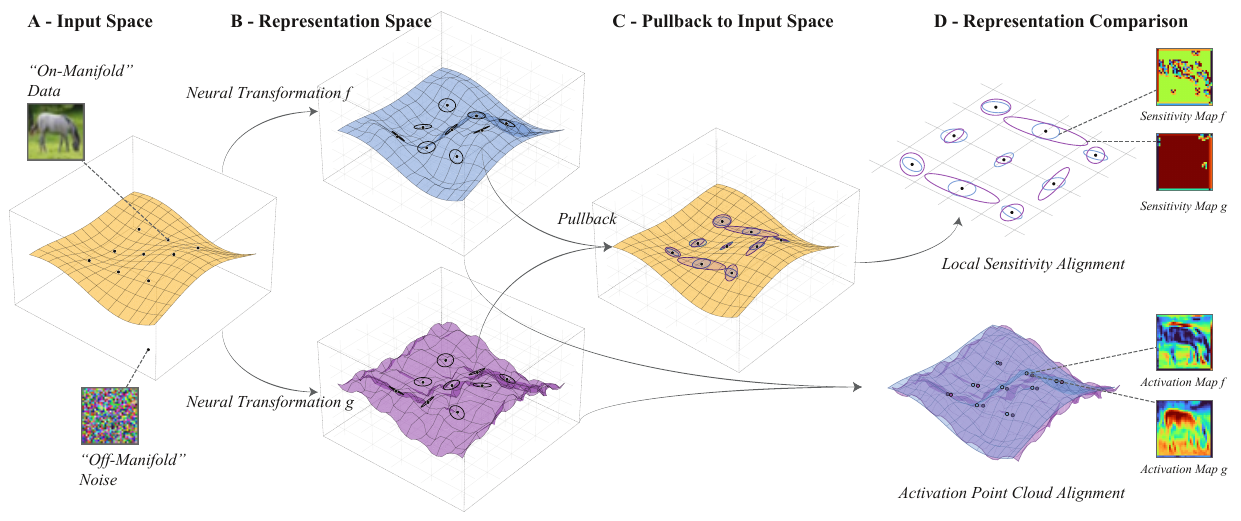}
    \caption{\textbf{Comparing representations via dataset-level sensitivity summaries.}
    (A) Input space contains structured data on a data manifold, contrasted with off-manifold noise.
    (B) Representation maps $f$ and $g$ transform the same data into feature spaces where local neighborhoods have different geometries.
    (C) Pulling these local geometries back to input/stimulus coordinates yields Jacobian-induced sensitivity ellipses, indicating which directions are amplified or suppressed.
    (D) We compare representations by averaging this local pullback/Fisher geometry over a specified perturbation family, rather than by aligning activation point clouds.}
      \label{fig:cartoon}
  \vspace{-0.9em}
\end{figure*}

We make three main contributions. First, we define second-moment local perturbation-discrimination tasks and identify the projected expected Fisher/pullback operator as their minimal complete dataset-level summary. Second, we compare regularized summaries with a log-spectral SPD distance, yielding \sras\ and a uniform multiplicative certificate over lifted task values. Third, we demonstrate that these summaries are empirically useful in both artificial and biological systems through a range of experiments, including layer correspondence, transferable class-conditional probes, robust-training dissociations, and Allen Brain Observatory static gratings recordings with explicitly defined stimulus-coordinate families.
\section{Local perturbation tasks and dataset-level sensitivity summaries}
\label{sec:theory}

We now turn the local-geometry viewpoint into a representation comparison. Starting from the Jacobian-based objects used in local-geometry methods \cite{Berardino2017EigenDistortions,Zhou2024Discriminability,Lipshutz2024LocalInfoGeometry,Feather2025PrincipalDistortions,cayco_gajic2026msa}, we specify which perturbations count, average their discriminability over stimuli, and identify the minimal summary for the resulting second-moment tasks.

\subsection{Representation maps and local discriminability}
A feedforward network can be written as a composition of layer-wise maps
\[
f_1,\dots,f_L,
\]
so that the representation at depth \(n\) is the input-to-layer function
\begin{equation}
\Phi_n(x)\coloneqq (f_n\circ f_{n-1}\circ \cdots \circ f_1)(x).
\end{equation}
This notation is standard when treating a representation as a function of the stimulus rather than as a collection of activation vectors alone \cite{lenc2015understanding}. It also makes clear that activation-based and Jacobian-based analyses begin from the same object. We therefore fix a layer and write
\[
f:\R^d\to\R^m
\]
for a differentiable representation map, with stimuli \(x\sim\mathcal D\). The corresponding pullback metric
\begin{equation}
M_f(x)\coloneqq J_f(x)^\top J_f(x)\in\mathbb S_+^d
\label{eq:pullback_metric}
\end{equation}
describes infinitesimal sensitivity: differentiability gives
\[
\|f(x+\epsilon v)-f(x)\|_2^2
=
\epsilon^2 v^\top M_f(x)v+o(\epsilon^2).
\]
This is the same local object used in recent intrinsic- and local-geometry approaches to representation comparison \citep{Lipshutz2024LocalInfoGeometry,Feather2025PrincipalDistortions,cayco_gajic2026msa}. Here it becomes a representation-level summary only after choosing a perturbation family and averaging discriminability over stimuli.

A standard operationalization of local sensitivity is discrimination between a stimulus and a small coordinate perturbation of it using noisy population activity \cite{Green1966SDT,DayanAbbott2001}. We formulate this in local stimulus coordinates. Let
\[
\psi:\Omega\subset\mathbb R^q\to\mathbb R^d,
\qquad
h_f=f\circ\psi,
\]
where \(q\) is the dimension of the stimulus-coordinate space and \(\psi\) parameterizes either pixel space, a data manifold, or an experimentally controlled stimulus family.

\begin{definition}[Local coordinate-perturbation discrimination task]
\label{def:local_discrimination_task}
Fix \(s\in\Omega\) and a small coordinate perturbation \(\eta\in\mathbb R^q\). An observer receives a noisy representation and must discriminate
\begin{equation}
H_0:\ r\sim \mathcal N(h_f(s),\Sigma)
\qquad\text{vs.}\qquad
H_1:\ r\sim \mathcal N(h_f(s+\eta),\Sigma),
\label{eq:local_discrimination_task}
\end{equation}
where \(\Sigma\in\mathbb S_{++}^m\) is a stimulus-independent noise covariance on the representation space.
\end{definition}

Linearizing $h_f$ gives
\[
h_f(s+\eta)\approx h_f(s)+J_{h_f}(s)\eta,
\qquad
J_{h_f}(s)=J_f(\psi(s))D\psi(s),
\]
so the local Fisher metric in the chosen stimulus coordinates is
\begin{equation}
I_{f,\psi}(s)
=
J_{h_f}(s)^\top \Sigma^{-1}J_{h_f}(s).
\label{eq:local_fisher}
\end{equation}
The corresponding squared discriminability and its second-moment average are
\begin{equation}
d'^2(s,\eta)=\eta^\top I_{f,\psi}(s)\eta,
\qquad
\E[d'^2\mid s]=\Tr(C I_{f,\psi}(s))
\quad
\text{for } C=\E[\eta\eta^\top],
\label{eq:local_discriminability}
\end{equation}
with derivation in Appendix~\ref{app:local_discriminability}. Thus local perturbation discriminability is governed by the same Fisher quadratic form used in population coding \cite{BrunelNadal1998FI,Pouget2000PopCodes}. In the isotropic case $\Sigma=\sigma^2 I$, $I_{f,\psi}(s)=\sigma^{-2}J_{h_f}(s)^\top J_{h_f}(s)$, so the Fisher geometry reduces to the pullback metric up to a scalar noise factor. For recurrent systems, \(f\) can denote either an unrolled fixed-time state map or, when dynamics converge, the equilibrium map \(x\mapsto r^*(x)\); Appendix~\ref{app:recurrence} gives the fixed-point form.

\subsection{Projected second-moment task families}
\label{sec:projected_tasks}

Pointwise discriminability becomes a representation-level comparison only after specifying the perturbation family and the stimulus distribution. We therefore fix a coordinate family, average over stimuli, and summarize the resulting second-moment discrimination tasks.

\begin{definition}[Projected second-moment local task family]
\label{def:projected_task_family}
Fix a smooth parameterization
\(\psi:\Omega\subset\mathbb R^q\to\mathbb R^d\), write \(h_f=f\circ\psi\), let
\(P\in\mathbb R^{q\times k}\) have orthonormal columns, and let \(s\sim\nu\).
Perturbations have the form \(s\mapsto s+Pz\), where \(z\in\mathbb R^k\) and
\(C_z=\E[zz^\top]\succeq0\). The projected expected Fisher operator is
\begin{equation}
F_{f,\psi,P,\Sigma}
\coloneqq
\E_{s\sim\nu}
\left[
P^\top J_{h_f}(s)^\top\Sigma^{-1}J_{h_f}(s)P
\right]
\in\mathbb S_+^k,
\label{eq:projected_fisher}
\end{equation}
and the associated expected local-discriminability task value is
\begin{equation}
T^{(\Sigma)}_{\psi,P,C_z}(f)
\coloneqq
\Tr(C_zF_{f,\psi,P,\Sigma}).
\label{eq:projected_task_value}
\end{equation}
When \(\psi\) is the identity or clear from context, we suppress it.
\end{definition}

By Eq.~\eqref{eq:local_discriminability},
\[
T^{(\Sigma)}_{\psi,P,C_z}(f)
=
\E_{s,z}\!\left[d'^2(s,Pz)\right].
\]
Thus \(\psi\), \(\operatorname{span}(P)\), and \(\Sigma\) define the local evidence family being compared.

\begin{proposition}[Task-summary property]
\label{prop:task_summary}
For fixed \(\psi\), \(P\), and \(\Sigma\succ0\), and for representation maps \(f\) and \(g\), the operator \(F_{f,\psi,P,\Sigma}\) determines all values
\(\{T^{(\Sigma)}_{\psi,P,C_z}(f):C_z\succeq0\}\). Conversely, if
\[
T^{(\Sigma)}_{\psi,P,C_z}(f)
=
T^{(\Sigma)}_{\psi,P,C_z}(g)
\qquad
\text{for all } C_z\succeq0,
\]
then
\[
F_{f,\psi,P,\Sigma}=F_{g,\psi,P,\Sigma}.
\]
\end{proposition}


The proof uses the PSD trace-separation argument in Appendix~\ref{app:trace_summary_proofs}. The substantive point is the identification of the task-family statistic: for the specified perturbation family and observation model, \(F_{f,\psi,P,\Sigma}\) is exactly the minimal complete dataset-level summary of expected local discriminability.
In the isotropic case \(\Sigma=\sigma^2I\),
\begin{equation}
F_{f,\psi,P,\Sigma}
=
\sigma^{-2}G_{f,\psi,P},
\qquad
G_{f,\psi,P}
\coloneqq
\E_s\left[(J_{h_f}(s)P)^\top(J_{h_f}(s)P)\right].
\label{eq:projected_epm}
\end{equation}
We call \(G_{f,\psi,P}\) the projected expected pullback metric. 

Only the subspace \(\operatorname{span}(P)\), not the particular orthonormal basis \(P\), matters for the task family; changing \(\psi\) or \(\operatorname{span}(P)\) changes the scientific question. The basis-invariance calculation is given in Appendix~\ref{app:basis_invariance}.

\subsection{Gain, shape, and shared contrast directions}
\label{sec:gain_shape_probes}

The same projected operator can be decomposed into total sensitivity, directional allocation, and class-conditional contrasts.

\begin{definition}[Gain, shape, and class-conditional contrast summaries]
\label{def:gain_shape_contrast}
For a $k$-dimensional coordinate family $P$, define
\begin{equation}
\gamma_{f,\psi,P}
\coloneqq
\frac{1}{k}\Tr(G_{f,\psi,P}),
\qquad
\widehat G_{f,\psi,P}
\coloneqq
\frac{G_{f,\psi,P}}{\Tr(G_{f,\psi,P})}
\quad
\text{when } \Tr(G_{f,\psi,P})>0 .
\label{eq:gain_shape}
\end{equation}
Thus $\gamma$ is total expected sensitivity, while $\widehat G$ records directional allocation. For normalized $C\succeq0$ with $\Tr(C)=1$, the shape-only task value is
\[
\widehat T_{\psi,P,C}(f)=\Tr(C\widehat G_{f,\psi,P}),
\]
which is nontrivial only for $k\ge2$.

For a restricted distribution $\nu_y$, such as images from class $y$, define
\begin{equation}
G^{(y)}_{f,\psi,P}
\coloneqq
\E_{s\sim\nu_y}
\left[(J_{h_f}(s)P)^\top(J_{h_f}(s)P)\right].
\label{eq:class_conditional_G}
\end{equation}
Given model groups $A$ and $B$, define
\begin{equation}
\Delta G_y
=
\overline G^{(y)}_{A,\psi,P}
-
\overline G^{(y)}_{B,\psi,P},
\qquad
\Delta\widehat G_y
=
\overline{\widehat G}^{(y)}_{A,\psi,P}
-
\overline{\widehat G}^{(y)}_{B,\psi,P}.
\label{eq:group_contrast}
\end{equation}
Here the overline denotes averaging the corresponding class-conditional summary over models in the indicated group. The full contrast captures gain and allocation differences; the shape-only contrast isolates allocation after trace normalization.

\end{definition}

\begin{proposition}[Optimal shared contrast directions]
\label{prop:shared_contrast_directions}
Let $\Delta G_y\in\mathbb S^k$ be the class-conditional group contrast in Definition~\ref{def:gain_shape_contrast}. Then, for every unit vector $v\in\mathbb R^k$,
\[
v^\top\Delta G_yv
\]
is the between-group contrast in expected local quadratic sensitivity along the shared coordinate-family direction $v$. Moreover,
\[
\max_{\|v\|_2=1}v^\top\Delta G_yv=\lambda_{\max}(\Delta G_y),
\qquad
\min_{\|v\|_2=1}v^\top\Delta G_yv=\lambda_{\min}(\Delta G_y),
\]
attained by top and bottom eigenvectors. The corresponding rank-$r$ statements follow from Ky Fan's principle, and the same results hold with $\Delta G_y$ replaced by $\Delta\widehat G_y$; see Appendix~\ref{app:proof:shared_contrast_directions}.
\end{proposition}

The contrast directions in Proposition~\ref{prop:shared_contrast_directions} are shared across the distribution defining $G^{(y)}$, unlike image-conditioned local distortion objectives. If $\Delta G_y(x)$ is an imagewise contrast and $q_x(v)=v^\top\Delta G_y(x)v$, then
\begin{equation}
\max_{\|v\|=1}\E_x[q_x(v)]
\le
\E_x\!\left[\max_{\|v\|=1}q_x(v)\right].
\label{eq:shared_vs_pointwise}
\end{equation}

The construction is therefore distribution-level: its directions are shared contrasts under \(G^{(y)}\), not separately optimized perturbations around each image. In the experiments below, we test whether these shared directions can be externalized into finite perturbation sets that transfer across held-out images and held-out models, making the construction complementary to eigen-distortions and principal distortions \cite{Berardino2017EigenDistortions,Lipshutz2024LocalInfoGeometry,Feather2025PrincipalDistortions}.

\section{\sras: geometry-aware comparison of sensitivity summaries}
\label{sec:sras}

Section~\ref{sec:projected_tasks} maps each representation and perturbation family to a PSD task summary. For a nonzero PSD summary \(A_{f,P}\in\mathbb S_+^k\), where \(A\) may denote either \(G\) or the noise-aware Fisher summary \(F\), we use the trace-scaled SPD lift
\begin{equation}
\widetilde A_{f,P}
\coloneqq
A_{f,P}
+
\epsreg\frac{\Tr(A_{f,P})}{k}I_k
\in\mathbb S_{++}^k .
\label{eq:spd_lift}
\end{equation}
When \(A=G\) we write \(\widetilde G\), and when \(A=F\) we write \(\widetilde F\). For lifted tasks, write \(\widetilde T_{P,C}(f)=\Tr(C\widetilde A_{f,P})\); the trace-summary argument from Proposition~\ref{prop:task_summary} applies unchanged.

We compare lifted signatures with the affine-invariant Riemannian metric on \(\mathbb S_{++}^k\) \cite{Pennec2006SPD},
\begin{equation}
\AIRM(A,B)
\coloneqq
\left\|
\log\!\left(A^{-1/2}BA^{-1/2}\right)
\right\|_F,
\label{eq:airm}
\end{equation}
and define, for a fixed perturbation family \(P\),
\begin{equation}
d_{\sras,P}(f,g)
\coloneqq
\frac{1}{\sqrt{k}}\AIRM(\widetilde G_{f,P},\widetilde G_{g,P}),
\qquad
\sras_P(f,g)
\coloneqq
\exp[-d_{\sras,P}(f,g)].
\label{eq:sras}
\end{equation}
The normalization by \(\sqrt{k}\) keeps scores comparable across family dimensions; in noise-aware settings we replace \(\widetilde G\) by \(\widetilde F\).
\begin{theorem}[Uniform control over lifted second-moment tasks]
\label{thm:uniform_task_control}
Let \(A,B\in\mathbb S_{++}^k\) and \(d=\AIRM(A,B)\). Then
\[
e^{-d}A\preceq B\preceq e^dA,
\]
and consequently, for every \(C\succeq0\),
\begin{equation}
e^{-d}\Tr(CA)
\le
\Tr(CB)
\le
e^d\Tr(CA).
\label{eq:uniform_task_control}
\end{equation}
\end{theorem}

Applying Theorem~\ref{thm:uniform_task_control} with \(A=\widetilde G_{f,P}\) and \(B=\widetilde G_{g,P}\) gives uniform multiplicative control over lifted task values. If
\[
\delta=d_{\sras,P}(f,g),
\qquad
S=\sras_P(f,g)=e^{-\delta},
\]
then
\begin{equation}
S^{\sqrt{k}}\widetilde T_{P,C}(f)
\le
\widetilde T_{P,C}(g)
\le
S^{-\sqrt{k}}\widetilde T_{P,C}(f).
\label{eq:sras_score_task_bound}
\end{equation}
Thus \(\sras\) is a dimension-normalized similarity score, while the certificate itself is controlled by the unnormalized AIRM distance. The task-control viewpoint also has a sharp worst-case form: replacing the Frobenius log-spectral distance in Theorem~\ref{thm:uniform_task_control} by its operator-norm analogue exactly captures the largest possible log-multiplicative disagreement over lifted task values (Appendix~\ref{app:dinfty}). Separately, because the construction depends on Jacobians rather than activation point clouds, it is insensitive to post-hoc representation-space shifts that are constant with respect to the stimulus, including the subset-translation construction of \cite{davari2022reliability} (Appendix~\ref{app:subset_translation}).
\section{Experiments}
\label{sec:experiments}
\paragraph{Experimental setup.}
We evaluate three model settings and one biological application: Tiny10 CNNs for layer matching; ResNet-18 \cite{he2016deep} and small-ViT \cite{dosovitskiy2020image} CIFAR-10 \cite{krizhevsky2009learning} banks for architecture comparisons; fixed-architecture ResNet-18 banks trained with ERM, PGD~\cite{madry2018towards}, TRADES~\cite{zhang2019trades}, or MART~\cite{wang2020mart} for regime comparisons; and Allen static-gratings recordings for biologically defined stimulus-coordinate families \((\theta,\rho,\phi)\). Activation baselines and sensitivity comparisons use the same checkpoints or neural-response summaries; details are in Appendix~\ref{app:exp_setup}.

\subsection{Layer matching examination}
\label{sec:layer_matching}

As a sanity check, we ask whether independently trained Tiny10 CNNs recover the expected layer correspondence under the retrieval protocol of \citet{kornblith2019cka}. A match is correct when the architecturally corresponding layer is the top row-wise and column-wise retrieval in the \(8\times8\) layer-similarity matrix.

Table~\ref{tab:layer_identification} summarizes the results. Under random perturbation families, \sras\ already gives a meaningful correspondence signal at \(K=16\) and improves with family dimension, reaching \(85.1\%\) at \(K=768\). Under PCA-basis families, it reaches \(93.8\%\), nearly matching linear and RBF CKA at \(94.4\%\) and \(94.0\%\), respectively. Thus the experiment is a coherence check: the sensitivity summaries recover layer identity and off-diagonal structure, while strong activation baselines remain best in raw accuracy.

We also include two pointwise local-geometry controls on the same \(n=1024\) benchmark and projected families: pointwise AIRM and a down-projected MSA-style baseline. These test the aggregation choice by contrasting dataset-level expected task-family operators with imagewise metric comparisons averaged over the dataset. pw-AIRM is strong at small \(K\) but degrades at larger \(K\), while full \sras\ gives the strongest high-\(K\) local-geometry result. The MSA-style control is near chance under random projections but remains above chance under PCA-basis families, indicating that the projected family can carry layer information even when the spectral-ratio summary is less discriminative than \sras. Appendix~\ref{app:sr_diagnosis} gives diagnostic context. Because these baselines are down-projected for tractability, we treat them as targeted controls rather than an exhaustive benchmark; Appendix Table~\ref{tab:layer_matching_compute} summarizes the resulting compute and memory scaling. Additional diagnostics are in Appendix~\ref{app:layer_matching}.

\begin{table*}[t]
\centering
\begingroup
\fontsize{6.3}{7.4}\selectfont
\setlength{\tabcolsep}{2.8pt}
\renewcommand{\arraystretch}{1.08}
\caption{
Layer-identification accuracy. All entries use the shared \(n=1024\)-image benchmark and report mean \(\pm\) SEM. Local-geometry cells are random / PCA-basis results. Compute/memory scaling is in Appendix Table~\ref{tab:layer_matching_compute}.
}
\label{tab:layer_identification}
\resizebox{0.99\textwidth}{!}{%
\begin{tabular}{@{}lccccc@{\hspace{0.8em}}l c@{}}
\toprule
& \multicolumn{5}{c}{\textbf{Local geometry: random / PCA basis}}
& \multicolumn{2}{c}{\textbf{Activation baselines}} \\
\cmidrule(lr){2-6}\cmidrule(l){7-8}
\textbf{Method}
& \(\mathbf{16}\)
& \(\mathbf{32}\)
& \(\mathbf{64}\)
& \(\mathbf{128}\)
& \textbf{Ref.}
& \textbf{Method}
& \textbf{Acc.} \\
\midrule

S-RAS (full)
& \(71.7{\pm}1.6 / 77.5{\pm}1.9\)
& \(72.6{\pm}1.6 / 78.8{\pm}1.8\)
& \(74.9{\pm}1.4 / 79.6{\pm}1.6\)
& \(76.5{\pm}1.2 / 81.5{\pm}1.5\)
& \(85.1{\pm}1.2 / \mathbf{93.8{\pm}0.9}^{\dagger}\)
& CKA (linear)
& \(\mathbf{94.4{\pm}0.9}^{\ddagger}\) \\

S-RAS (shape)
& \(49.2{\pm}1.3 / 57.1{\pm}1.9\)
& \(53.1{\pm}1.6 / 62.8{\pm}1.9\)
& \(54.0{\pm}1.6 / 65.0{\pm}1.9\)
& \(54.2{\pm}1.8 / 67.4{\pm}2.1\)
& \(55.7{\pm}1.7 / 78.1{\pm}1.6\)
& CKA (RBF)
& \(94.0{\pm}0.8\) \\

pw-AIRM
& \(73.8{\pm}1.6 / 83.2{\pm}1.3\)
& \(73.5{\pm}1.5 / 82.5{\pm}1.2\)
& \(60.4{\pm}1.3 / 75.8{\pm}1.1\)
& \(50.4{\pm}1.0 / 56.9{\pm}1.3\)
& \(45.6{\pm}1.1 / 46.3{\pm}1.1\)
& Procrustes
& \(66.0{\pm}0.8\) \\

MSA\(^*\)
& \(15.3{\pm}0.6 / 39.9{\pm}1.3\)
& \(13.8{\pm}0.4 / 35.0{\pm}1.1\)
& \(14.3{\pm}0.5 / 32.5{\pm}1.0\)
& \(14.0{\pm}0.5 / 31.5{\pm}0.6\)
& \(14.2{\pm}0.5 / 29.2{\pm}0.7\)
& CCA (\(R^2\))
& \(12.5{\pm}0.0\) \\

\bottomrule
\end{tabular}%
}
\par\vspace{0.2em}
\begin{minipage}{0.99\textwidth}
\footnotesize
Ref. is \(K=768\) for S-RAS and \(K=256\) for pointwise controls.
\(^\dagger\) Best local-geometry result. 
\(^\ddagger\) Best activation-based and overall result.
\(^*\) Exact local spectral-ratio distance averaged pointwise over images on family-restricted projected metrics.
\end{minipage}
\endgroup
\vspace{-0.5em}
\end{table*}

\subsection{Held-out class-conditional diagnostic probes}
\label{sec:diagnostic_probes}

\begin{figure*}[t]
    \centering
    \includegraphics[width=\textwidth]{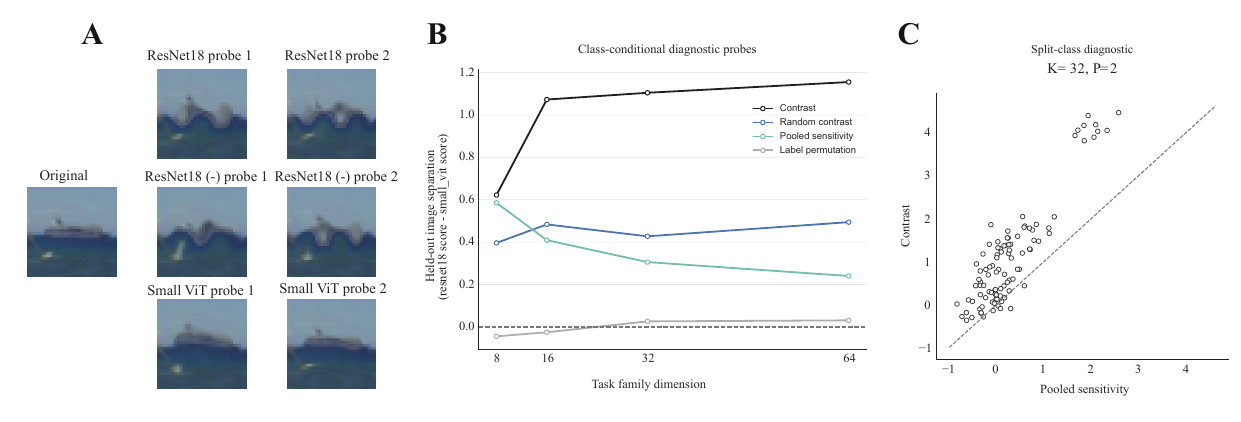}
    \caption{
\textbf{Held-out class-conditional diagnostic probes.}
(A) Example probes at \(K=32\), two probes per side.
(B) Mean held-out image separation across task-family dimensions and controls.
(C) Split-class diagnostic visualization against pooled sensitivity; formal tests aggregate to the held-out model-split level.
}
    \label{fig:diagnostic_probes_main}
\end{figure*}

Layer matching tests whether sensitivity summaries recover a basic correspondence signal. We next ask whether class-conditional sensitivity contrasts can be externalized into finite perturbations whose effects transfer to held-out images and held-out models. This instantiates Section~\ref{sec:gain_shape_probes}: a learned smooth-deformation parameterization defines a \(q\)-dimensional coordinate system, \(P_K\) selects a \(K\)-dimensional task-family subspace, and leading positive and negative eigendirections of the discovery-set contrast are instantiated as finite image deformations.

For each class \(y\), we estimate the summaries in Eq.~\eqref{eq:class_conditional_G}, form the discovery-set contrast in Eq.~\eqref{eq:group_contrast}, and evaluate only on held-out images and models. The probes optimize between-group local quadratic sensitivity, not pooled sensitivity shared by both groups. We compare them with random-contrast probes, pooled-sensitivity probes from \(G_A^{(y)}+G_B^{(y)}\), and a label-permutation null; construction details, split protocol, amplitudes, held-out model AUC, and ablations are in Appendix~\ref{app:diagnostic_probes}.

At the main setting (\(K=32\), two probes per side), contrast-derived probes produce larger held-out image separation than all controls: \(1.10\), compared with \(0.43\) for random contrast, \(0.31\) for pooled sensitivity, and \(0.03\) for the label-permutation null. After aggregating over classes within each held-out model split, contrast-derived probes outperform all three controls in all \(10\) splits, giving one-sided paired Wilcoxon \(p=9.77\times10^{-4}\) for each comparison. The effect is present by \(K=16\), stable through \(K=64\), and robust across probe counts at \(K=32\). We therefore treat held-out image transfer of finite probes, rather than held-out model classification, as the primary claim.

\subsection{Robust training as a controlled dissociation}
\label{sec:robust_dissociation}

We next apply the same contrast-and-probe construction to CIFAR-10 ResNet-18 banks trained with ERM, PGD, TRADES, or MART. Because architecture is fixed, separations cannot be attributed to model class. We use the pre-logit representation and \(K=32\) as the default family size.

\begin{figure*}[t]
    \centering
    \includegraphics[width=\textwidth]{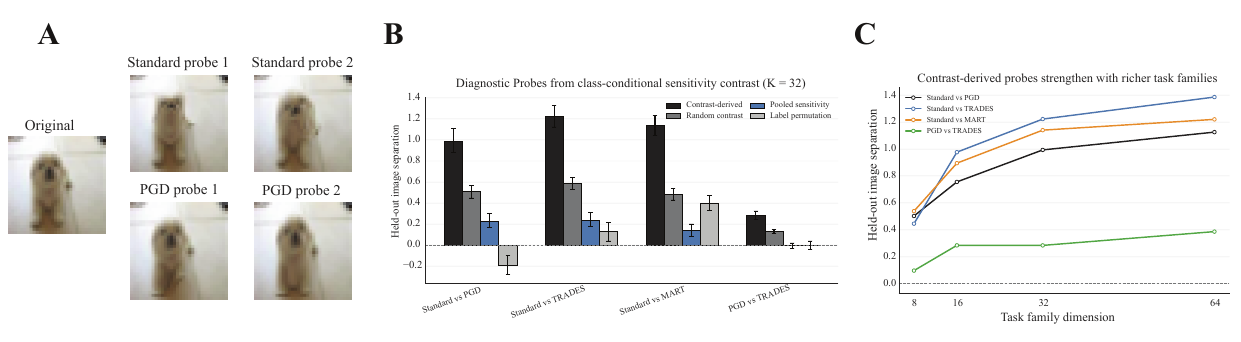}
    \caption{
    \textbf{Controlled regime shifts under robust training.}
    (A) Example probes from the standard-versus-PGD contrast.
    (B) Held-out image separation at \(K=32\), comparing contrast-derived probes to random-contrast, pooled-sensitivity, and label-permutation controls. Error bars in (B) are 95\% percentile bootstrap intervals over held-out model-split means.
    (C) Separation as a function of task-family dimension \(K\).
    }
    \label{fig:adv_std_main}
\end{figure*}

Figure~\ref{fig:adv_std_main} shows that contrast-derived probes are smooth, finite perturbations and that they separate standard models from PGD, TRADES, and MART on held-out images, with mean separations \(0.99\), \(1.22\), and \(1.14\), respectively. In all three standard-versus-robust comparisons, contrast-derived probes outperform the three controls at \(K=32\) (one-sided Wilcoxon \(p=9.77\times10^{-4}\) over held-out splits). The same construction also separates PGD from TRADES more weakly (\(0.28\)), and the effect is already present by \(K=16\) and generally strengthens through \(K=32\) and \(64\).

Appendix~\ref{app:robust_dissociation} extends the analysis to all six regime pairs, layer ablations, and direct model-space separations. There, TRADES versus MART shows a smaller but reliable effect, whereas PGD versus MART shows little separation under the present family. This selectivity is informative: the probe battery is not simply flagging every regime pair, but is strongest where training objectives induce clearer changes in local evidence use.

\subsection{Biological static-gratings application}
\label{sec:biology_static_gratings}

We next instantiate the theory in an experimentally parameterized stimulus space. The Allen Brain Observatory Visual Coding dataset provides a standardized survey of visually evoked calcium responses across mouse visual cortical areas and depths \cite{deVries2020AllenBO}. We focus on static gratings, whose controlled variables define the \(q=3\) coordinate system
\[
s=(\theta,\rho,\phi),
\qquad
\rho=\log_2(\text{spatial frequency}),
\]
with \(6\) orientations, \(5\) spatial frequencies, and \(4\) phases, yielding \(120\) conditions \cite{AllenStaticGratings}. This setting matches the parameterized-family version of Section~\ref{sec:theory} where the perturbation family is a biologically meaningful stimulus-coordinate family rather than an arbitrary pixel subspace. It also connects to population-coding theory, where Fisher information and correlated variability determine local stimulus discriminability \cite{SeungSompolinsky1993,AbbottDayan1999,Ecker2011NoiseCorrelations}, and to neural-manifold analyses of visual cortex population responses \cite{Stringer2019HighDim,Beshkov2024Topology}.

For each experiment \(e\), we estimate condition-mean responses \(\mu_e(s)\in\mathbb R^{n_{\mathrm{match}}}\), pooled within-condition covariance \(\Sigma_e\), and finite-difference response Jacobian \(D\mu_e(s)\). The noise-aware Fisher summary and unwhitened local-geometry baseline are
\[
F_e
=
\E_s\!\left[
D\mu_e(s)^\top \Sigma_e^{-1}D\mu_e(s)
\right],
\qquad
G_e
=
\E_s\!\left[
D\mu_e(s)^\top D\mu_e(s)
\right].
\]
We also compare trace-normalized shape-only versions, \(\widehat F_e=F_e/\Tr(F_e)\) and \(\widehat G_e=G_e/\Tr(G_e)\), to separate total sensitivity scale from directional allocation. Full preprocessing, covariance regularization, finite differences, family restrictions, and donor-distinct retrieval details are in Appendix~\ref{app:biology_details}.

\begin{figure*}[t]
    \centering
    \includegraphics[width=\textwidth]{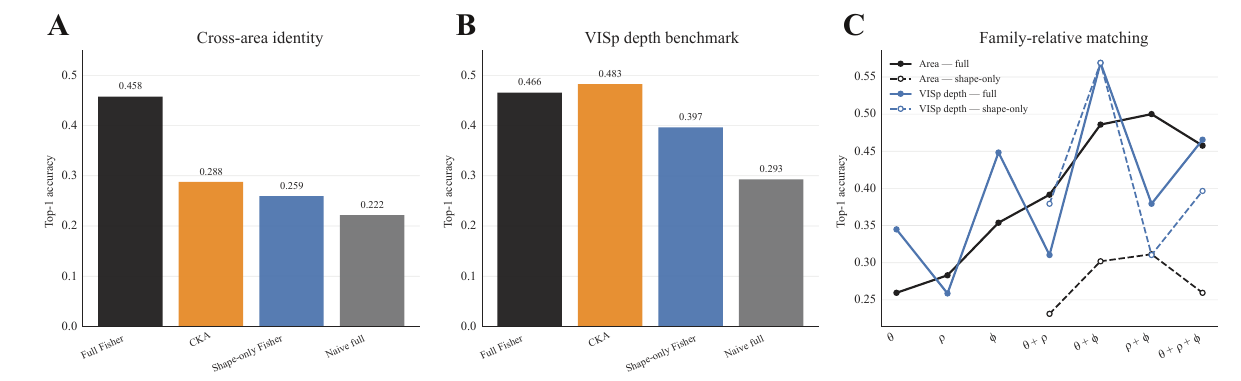}
    \caption{
    \textbf{Biological sensitivity geometry in mouse visual cortex.}
    (A) Cross-area identity under donor-distinct matched-count comparison.
    (B) VISp-depth identity is harder. Full Fisher remains competitive with CKA and above naive geometry.
    (C) Matching depends on the grating-coordinate family. Dashed curves show trace-normalized shape-only summaries for families with dimension at least two.
    }
    \label{fig:biology_main}
\end{figure*}

\begin{table}[t]
\centering
\begingroup
\fontsize{7.0}{8.0}\selectfont
\setlength{\tabcolsep}{3.0pt}
\renewcommand{\arraystretch}{1.08}
\caption{\textbf{Biological top-1 retrieval.} Entries are donor-distinct top-1 identity accuracy after matched-count control; reliability is mean split-half reliability after matching.}
\label{tab:biology_top1}
\resizebox{0.99\linewidth}{!}{%
\begin{tabular}{@{}lccccrcc@{}}
\toprule
Cohort & \(n_{\mathrm{match}}\) & Reliability & Chance
& Full Fisher & CKA & Shape-only Fisher & Naive full \\
\midrule
Area & 60 & 0.941 & 0.167 & \textbf{0.458} & 0.288 & 0.259 & 0.222 \\
VISp depth & 45 & 0.932 & 0.250 & 0.466 & \textbf{0.483} & 0.397 & 0.293 \\
\bottomrule
\end{tabular}%
}
\endgroup
\end{table}

We evaluate donor-distinct retrieval in cross-area and within-\(\mathrm{VISp}\) depth cohorts, matching counts before retrieval because Fisher scale depends on population size. In the primary area benchmark, full Fisher reaches \(0.458\) top-1 accuracy, above CKA \(0.288\), shape-only Fisher \(0.259\), and naive full geometry \(0.222\) (Table~\ref{tab:biology_top1}, Figure~\ref{fig:biology_main}A). Its advantage over CKA is significant under paired query-level permutation testing after FDR correction (\(q=8.0\times10^{-4}\)), as is the advantage over naive full geometry. Thus the area-level signal is not explained by activation alignment or unwhitened derivative geometry alone; the covariance term is part of the comparison, consistent with the role of correlated variability in population-code accuracy \cite{SeungSompolinsky1993,AbbottDayan1999,Ecker2011NoiseCorrelations}.

The VISp-depth benchmark is more nuanced. CKA reaches \(0.483\) top-1 accuracy, full Fisher \(0.466\), shape-only Fisher \(0.397\), and naive full \(0.293\). On diagonal AUC, full Fisher and CKA are nearly tied, \(0.549\) and \(0.547\), while naive full remains lower at \(0.513\) (Appendix Table~\ref{tab:biology_diag_auc}). Thus VISp depth identity is a harder within-area comparison, where the noise-aware Fisher summary remains informative but no longer separates as sharply from activation-based summaries.

For \(Q\subseteq\{\theta,\rho,\phi\}\), we compare \(F_{e,Q}=P_Q^\top F_eP_Q\) and, for \(|Q|\ge2\), its trace-normalized version. The best-matching family is not always the full coordinate set; in the area cohort, \(\rho+\phi\) reaches \(0.500\) top-1 accuracy and \(\theta+\phi\) reaches \(0.486\), both above the full \(\theta+\rho+\phi\) family at \(0.458\); in VISp depth, \(\theta+\phi\) reaches \(0.569\), also exceeding the full family. Thus adding coordinates can dilute the matching signal when the added direction emphasizes variation irrelevant to the population distinction under study. The gain/shape decomposition adds a complementary view where area identity depends on both total noise-aware sensitivity and directional allocation, while shape-only Fisher gives the strongest area diagonal-AUC and adds information beyond activation, decoder-profile, and mapping baselines in an auxiliary pair-classification analysis (Appendix Table~\ref{tab:biology_incremental}).

\section{Discussion}

Activation-based comparisons are most informative when the question concerns stimulus-evoked organization or downstream readout structure. We asked a complementary question: whether two representations make the same small stimulus directions discriminable on average over a dataset. We formalized this through second-moment local perturbation tasks, used the projected expected pullback/Fisher operator as the minimal complete dataset-level summary for that family, and compared regularized summaries with a log-spectral SPD geometry.

The resulting similarity is family-relative rather than universal: a perturbation family \(P\) defines which local evidence is under study, so agreement under one family need not imply agreement under another. This dependence appears across the experiments: perturbation-family choice affects layer correspondence and probe transfer in models, while Allen static-gratings recordings show that noise-aware Fisher geometry and grating-coordinate choice affect area and depth matching.

Empirically, \sras\ recovers meaningful layer-correspondence structure and sharper off-diagonal suppression, though not uniformly higher raw top-match accuracy than strong activation baselines. The stronger evidence comes from targeted diagnostics: class-conditional probes transfer to held-out images and models, robust-training contrasts reveal regime-specific evidence-use differences, and static-gratings recordings expose family-dependent biological sensitivity structure.

The framework remains local, dataset-level, and second-moment based; it does not capture image-conditioned optima, higher-order effects, or strongly nonlinear finite-perturbation behavior. Its discriminability interpretation depends on the stated observation model, especially stimulus-independent noise in the Fisher form. Broader benchmarks, especially with larger architectures using subsampling, lower-rank families, and distributed JVP batching, are needed to identify when dataset-level sensitivity summaries reveal evidence-use structure not captured by activation alignment or pointwise local-geometry baselines.

\section*{Code and Data Availability}
Code for reproducing the main analyses is available at \url{https://github.com/amirhosseinyavari/neural-sensitivity-geometry}. The Allen Brain Observatory Visual Coding data are publicly available from the Allen Institute.
\FloatBarrier
\FloatBarrier

\bibliographystyle{plainnat}
\bibliography{ref}
\appendix

\section{Theory details and proofs}
\label{app:appendix}

\subsection{Derivation of local discriminability}
\label{app:local_discriminability}

For the Gaussian discrimination task in Eq.~\eqref{eq:local_discrimination_task}, the two hypotheses share covariance $\Sigma$ and differ only in their means,
\[
\mu_0=h_f(s),
\qquad
\mu_1=h_f(s+\eta).
\]
For equal-covariance Gaussian discrimination, the squared discriminability is
\[
d'^2
=
(\mu_1-\mu_0)^\top\Sigma^{-1}(\mu_1-\mu_0).
\]
Under local linearization,
\[
h_f(s+\eta)
\approx
h_f(s)+J_{h_f}(s)\eta,
\]
so
\[
d'^2(s,\eta)
=
\eta^\top J_{h_f}(s)^\top\Sigma^{-1}J_{h_f}(s)\eta
=
\eta^\top I_{f,\psi}(s)\eta,
\]
with $I_{f,\psi}(s)$ defined in Eq.~\eqref{eq:local_fisher}. If $C=\E[\eta\eta^\top]$, then
\[
\E[d'^2\mid s]
=
\E[\eta^\top I_{f,\psi}(s)\eta\mid s]
=
\Tr(CI_{f,\psi}(s)),
\]
which gives Eq.~\eqref{eq:local_discriminability}. For the projected family $\eta=Pz$ with $C_z=\E[zz^\top]$,
\[
\E_{s,z}[d'^2(s,Pz)]
=
\Tr\!\left(
C_z
\E_s[
P^\top J_{h_f}(s)^\top\Sigma^{-1}J_{h_f}(s)P]
\right)
=
\Tr(C_zF_{f,\psi,P,\Sigma}),
\]
which is the task value in Eq.~\eqref{eq:projected_task_value}.

\subsection{Trace-task summaries and minimality}
\label{app:trace_summary_proofs}

All completeness/minimality claims used in the main text are instances of the same elementary separation fact.

\begin{lemma}[PSD trace probes separate symmetric matrices]
\label{lem:psd_trace_separation}
Let $D\in\mathbb S^k$. If
\[
\Tr(CD)=0
\qquad
\text{for all } C\succeq0,
\]
then $D=0$.
\end{lemma}

\begin{proof}
If $D\neq0$, then some vector $v$ satisfies $v^\top Dv\neq0$. Taking $C=vv^\top\succeq0$ gives
\[
0=\Tr(vv^\top D)=v^\top Dv,
\]
a contradiction.
\end{proof}

\begin{lemma}[Linear trace-task summary]
\label{lem:linear_trace_summary}
Let $\Gamma_f,\Gamma_g\in\mathbb S^k$ and define
\[
T_C(f)=\Tr(C\Gamma_f),
\qquad
T_C(g)=\Tr(C\Gamma_g),
\qquad
C\succeq0.
\]
Then $\Gamma_f$ determines all task values $\{T_C(f):C\succeq0\}$. Conversely, if $T_C(f)=T_C(g)$ for all $C\succeq0$, then $\Gamma_f=\Gamma_g$.
\end{lemma}

\begin{proof}
Completeness is immediate from the definition of $T_C$. For minimality, equality of all task values gives
\[
\Tr(C(\Gamma_f-\Gamma_g))=0
\qquad
\text{for all } C\succeq0.
\]
Lemma~\ref{lem:psd_trace_separation} implies $\Gamma_f=\Gamma_g$.
\end{proof}

\paragraph{Projected Fisher and pullback summaries.}
Proposition~\ref{prop:task_summary} follows from Lemma~\ref{lem:linear_trace_summary} by taking
\[
\Gamma_f=F_{f,\psi,P,\Sigma}.
\]
The isotropic pullback case follows from
\[
F_{f,\psi,P,\sigma^2I}=\sigma^{-2}G_{f,\psi,P}.
\]

\paragraph{Lifted summaries.}
The same argument applies to
\[
\widetilde T_{P,C}(f)=\Tr(C\widetilde A_{f,P})
\]
by taking \(\Gamma_f=\widetilde A_{f,P}\), with \(A=G\) or \(A=F\).

\paragraph{Trace-normalized shape summaries.}
For shape-only summaries, assume $\Tr(G_{f,\psi,P})>0$ and define
\[
\widehat G_{f,\psi,P}
=
\frac{G_{f,\psi,P}}{\Tr(G_{f,\psi,P})}.
\]
For normalized probes $C\succeq0$ with $\Tr(C)=1$,
\[
\widehat T_{\psi,P,C}(f)=\Tr(C\widehat G_{f,\psi,P}).
\]
If $\widehat T_{\psi,P,C}(f)=\widehat T_{\psi,P,C}(g)$ for all such $C$, then by rescaling nonzero PSD matrices, the equality holds for all $C\succeq0$. Lemma~\ref{lem:linear_trace_summary} gives
\[
\widehat G_{f,\psi,P}
=
\widehat G_{g,\psi,P}.
\]

\subsection{Proof of Theorem~\ref{thm:uniform_task_control}}
\label{app:proof_uniform_task_control}

\begin{proof}
Let
\[
M=A^{-1/2}BA^{-1/2}\in\mathbb S_{++}^k.
\]
If $\lambda_i(M)$ are the eigenvalues of $M$, then
\[
\AIRM(A,B)^2
=
\|\log M\|_F^2
=
\sum_{i=1}^k(\log\lambda_i(M))^2.
\]
Thus $|\log\lambda_i(M)|\le d$ for every $i$, where $d=\AIRM(A,B)$. Hence
\[
e^{-d}I\preceq M\preceq e^dI.
\]
Congruence by $A^{1/2}$ gives
\[
e^{-d}A\preceq B\preceq e^dA.
\]
Now let $C\succeq0$. Congruence by $C^{1/2}$ preserves the Loewner order, so
\[
e^{-d}C^{1/2}AC^{1/2}
\preceq
C^{1/2}BC^{1/2}
\preceq
e^dC^{1/2}AC^{1/2}.
\]
Taking traces and using cyclicity gives
\[
e^{-d}\Tr(CA)
\le
\Tr(CB)
\le
e^d\Tr(CA).
\]
\end{proof}

\subsection{Basis invariance of projected task families}
\label{app:basis_invariance}

Let \(P'=PQ\), where \(Q^\top Q=I_k\). Then
\[
F_{f,\psi,P',\Sigma}
=
Q^\top F_{f,\psi,P,\Sigma}Q,
\qquad
G_{f,\psi,P'}
=
Q^\top G_{f,\psi,P}Q.
\]
For the corresponding covariance change \(C_z'=Q^\top C_zQ\), the task value is unchanged:
\[
\Tr(C_z'F_{f,\psi,P',\Sigma})
=
\Tr(C_zF_{f,\psi,P,\Sigma}),
\]
and likewise in the isotropic pullback case,
\[
\Tr(C_z'G_{f,\psi,P'})
=
\Tr(C_zG_{f,\psi,P}).
\]
Thus the projected task family depends on the chosen coordinate subspace, not on the particular orthonormal basis used to represent it.

\subsection{Tight worst-case constant and relation to AIRM}
\label{app:dinfty}

The AIRM-based bound in Theorem~\ref{thm:uniform_task_control} is convenient and symmetric, but in some settings one may want the sharpest possible worst-case log-spectral constant. This leads to the operator-norm quantity
\[
\dinf(A,B)
\coloneqq
\left\|
\log\!\left(A^{-1/2} B A^{-1/2}\right)
\right\|_{\op}.
\]

\begin{theorem}[Tight worst-case constant via \(\dinf\)]
\label{thm:dinf_tight}
Let \(A,B \in \Spp^k\). Then \(\dinf(A,B)\) is the smallest constant \(t \ge 0\) such that
\begin{equation}
e^{-t}A \preceq B \preceq e^{t}A.
\label{eq:dinf_rayleigh}
\end{equation}
\end{theorem}

\begin{lemma}[\(\dinf\) and AIRM]
\label{lem:dinf_airm}
For all \(A,B \in \Spp^k\),
\[
\dinf(A,B)\le \AIRM(A,B)\le \sqrt{k}\,\dinf(A,B).
\]
\end{lemma}

\subsection{Proof of Theorem~\ref{thm:dinf_tight}}
\label{app:proof:dinf_tight}

\begin{proof}
Write
\[
C \coloneqq A^{-1/2} B A^{-1/2} \in \Spp^k .
\]
Then the sandwich
\[
e^{-t}A \preceq B \preceq e^{t}A
\]
is equivalent, after congruence by \(A^{-1/2}\), to
\[
e^{-t}I \preceq C \preceq e^{t}I .
\]
Because \(C\) is symmetric positive definite, this holds if and only if every eigenvalue of \(C\) lies in the interval \([e^{-t},e^{t}]\), that is,
\[
e^{-t} \le \lambda_i(C) \le e^{t}
\qquad\text{for all } i .
\]
Equivalently,
\[
\bigl|\log \lambda_i(C)\bigr| \le t
\qquad\text{for all } i .
\]
Since \(\log C\) is symmetric, the smallest such \(t\) is precisely the largest absolute eigenvalue of \(\log C\), namely
\[
t = \|\log C\|_{\op} = \dinf(A,B).
\]
This proves the claim.
\end{proof}

\begin{theorem}[Variational characterization of \(\dinf\)]
\label{thm:dinf_variational}
For all \(A,B\in\Spp^k\),
\begin{equation}
\dinf(A,B)
=
\sup_{C\succeq 0,\ C\neq 0}
\left|
\log \frac{\Tr(CB)}{\Tr(CA)}
\right|.
\label{eq:dinf_variational}
\end{equation}
Equivalently,
\begin{equation}
\dinf(A,B)
=
\sup_{C\succeq 0,\ \Tr(CA)=1}
\left|
\log \Tr(CB)
\right|.
\label{eq:dinf_variational_normalized}
\end{equation}
\end{theorem}

\subsection{Proof of Theorem~\ref{thm:dinf_variational}}
\label{app:proof:dinf_variational}

\begin{proof}
Let
\[
M\coloneqq A^{-1/2}BA^{-1/2}\in \Spp^k.
\]
For any \(C\succeq 0\) with \(C\neq 0\), define
\[
D\coloneqq A^{1/2} C A^{1/2}\succeq 0.
\]
Then
\[
\Tr(CB)=\Tr(DM),
\qquad
\Tr(CA)=\Tr(D).
\]
Therefore
\begin{equation}
\frac{\Tr(CB)}{\Tr(CA)}=\frac{\Tr(DM)}{\Tr(D)}.
\label{eq:ratio_DM}
\end{equation}

Let the eigenvalues of \(M\) be
\[
\lambda_{\min}(M)\le \lambda_1(M)\le \cdots \le \lambda_{\max}(M).
\]
Since \(D\succeq 0\), the quantity \(\Tr(DM)/\Tr(D)\) is a convex combination of the eigenvalues of \(M\). Hence
\[
\lambda_{\min}(M)
\le
\frac{\Tr(DM)}{\Tr(D)}
\le
\lambda_{\max}(M).
\]
Applying \(\log\) gives
\[
\log \lambda_{\min}(M)
\le
\log \frac{\Tr(CB)}{\Tr(CA)}
\le
\log \lambda_{\max}(M).
\]
Therefore
\[
\left|
\log \frac{\Tr(CB)}{\Tr(CA)}
\right|
\le
\max\!\bigl\{|\log \lambda_{\min}(M)|,\ |\log \lambda_{\max}(M)|\bigr\}
=
\|\log M\|_{\op}
=
\dinf(A,B).
\]

To show equality, let \(u_{\max}\) be a unit eigenvector of \(M\) with eigenvalue \(\lambda_{\max}(M)\). Set
\[
D=u_{\max}u_{\max}^\top,
\qquad
C=A^{-1/2} D A^{-1/2}.
\]
Then \(C\succeq 0\), \(C\neq 0\), and
\[
\frac{\Tr(CB)}{\Tr(CA)}
=
\frac{\Tr(DM)}{\Tr(D)}
=
\lambda_{\max}(M).
\]
Thus
\[
\left|
\log \frac{\Tr(CB)}{\Tr(CA)}
\right|
=
|\log \lambda_{\max}(M)|.
\]
Likewise, choosing an eigenvector for \(\lambda_{\min}(M)\) attains
\[
|\log \lambda_{\min}(M)|.
\]
Taking the larger of the two proves \eqref{eq:dinf_variational}.

The normalized form \eqref{eq:dinf_variational_normalized} follows by rescaling \(C\) so that \(\Tr(CA)=1\).
\end{proof}

\subsection{Proof of Lemma~\ref{lem:dinf_airm}}
\label{app:proof:dinf_airm}

\begin{proof}
Let
\[
L \coloneqq \log\!\left(A^{-1/2} B A^{-1/2}\right).
\]
Then, by definition,
\[
\dinf(A,B)=\|L\|_{\op},
\qquad
\AIRM(A,B)=\|L\|_F.
\]
Since \(L\) is symmetric, \(\|L\|_{\op}\) is the \(\ell_\infty\) norm of its eigenvalues in absolute value, whereas \(\|L\|_F\) is the corresponding \(\ell_2\) norm. The standard relation between these norms yields
\[
\|L\|_{\op} \le \|L\|_F \le \sqrt{k}\,\|L\|_{\op},
\]
and hence
\[
\dinf(A,B)\le \AIRM(A,B)\le \sqrt{k}\,\dinf(A,B).
\]
\end{proof}

\subsection{Immunity of Jacobian-based alignment to subset translation}
\label{app:subset_translation}

\begin{lemma}[Invariance to locally constant representation-space shifts]
\label{lem:subset_translation}
Fix a finite dataset \(\{x_i\}_{i=1}^n\) and a differentiable representation map \(f:\R^d\to\R^m\). Suppose there are open neighborhoods \(U_i\) of \(x_i\) and constants \(c_i\in\R^m\) such that
\[
\tilde f(x)=f(x)+c_i
\qquad
\text{for all } x\in U_i .
\]
Then, at each dataset point,
\begin{equation}
J_{\tilde f}(x_i)=J_f(x_i),
\qquad
M_{\tilde f}(x_i)=M_f(x_i).
\end{equation}
Consequently, if \(G_f\) is computed by averaging over the empirical dataset, then
\begin{equation}
G_{\tilde f}=G_f,
\end{equation}
and any score depending on \(f\) only through \(G_f\) is unchanged.
\end{lemma}

\begin{proof}
On each neighborhood \(U_i\), the added term \(c_i\) is constant with respect to the input coordinates. Hence its derivative is zero on \(U_i\), and in particular
\[
J_{\tilde f}(x_i)=J_f(x_i).
\]
Therefore
\[
M_{\tilde f}(x_i)
=
J_{\tilde f}(x_i)^\top J_{\tilde f}(x_i)
=
J_f(x_i)^\top J_f(x_i)
=
M_f(x_i).
\]
Averaging over the empirical distribution gives \(G_{\tilde f}=G_f\). Any score constructed only from \(G_f\), including \(\sras\) after the SPD lift, is therefore unchanged.
\end{proof}

\begin{remark}
This invariance applies to post-hoc representation-space shifts that are locally constant with respect to input coordinates on neighborhoods of the evaluated datapoints. If the added shift varies smoothly with \(x\), then its derivative need not vanish and the sensitivity signature may change.
\end{remark}

\subsection{Structural diagnostics for pointwise spectral-ratio baselines}
\label{app:sr_diagnosis}

For the pointwise MSA baseline used in the layer-matching appendix, we write the local spectral-ratio distance as
\begin{equation}
d_{\mathrm{SR}}(A,B)
\coloneqq
1-\sqrt{\frac{\lambda_{\min}}{\lambda_{\max}}},
\label{eq:sr_distance_appendix}
\end{equation}
where \(\lambda_{\min}\) and \(\lambda_{\max}\) are the smallest and largest generalized eigenvalues of the pair \((B,A)\), equivalently the smallest and largest eigenvalues of \(A^{-1/2}BA^{-1/2}\).

\begin{proposition}[Conformal blindness of spectral-ratio pointwise similarity]
\label{prop:sr_conformal_blindness}
Let \(A,B\in \Spp^k\). If
\[
B=cA
\qquad
\text{for some } c>0,
\]
then
\[
d_{\mathrm{SR}}(A,B)=0.
\]
\end{proposition}

\begin{proof}
If \(B=cA\), then
\[
A^{-1/2}BA^{-1/2}
=
A^{-1/2}(cA)A^{-1/2}
=
cI.
\]
Hence all generalized eigenvalues are equal:
\[
\lambda_{\min}=\lambda_{\max}=c.
\]
Substituting into \eqref{eq:sr_distance_appendix} gives
\[
d_{\mathrm{SR}}(A,B)
=
1-\sqrt{\frac{c}{c}}
=
0.
\]
\end{proof}

\begin{proposition}[Extreme-spectrum collapse]
\label{prop:sr_extreme_spectrum}
Let \(A,B,\widetilde A,\widetilde B\in \Spp^k\). If
\[
\kappa\!\left(A^{-1/2}BA^{-1/2}\right)
=
\kappa\!\left(\widetilde A^{-1/2}\widetilde B\widetilde A^{-1/2}\right),
\]
where
\[
\kappa(M)\coloneqq \frac{\lambda_{\max}(M)}{\lambda_{\min}(M)},
\]
then
\[
d_{\mathrm{SR}}(A,B)=d_{\mathrm{SR}}(\widetilde A,\widetilde B).
\]
In particular, \(d_{\mathrm{SR}}\) depends only on the extreme generalized eigenvalues of the relative metric and ignores the interior spectrum.
\end{proposition}

\begin{proof}
By definition,
\[
d_{\mathrm{SR}}(A,B)
=
1-\sqrt{\frac{\lambda_{\min}}{\lambda_{\max}}}
=
1-\kappa\!\left(A^{-1/2}BA^{-1/2}\right)^{-1/2}.
\]
Thus equal relative condition number implies equal spectral-ratio distance.
\end{proof}

\subsection{Proof of Proposition~\ref{prop:shared_contrast_directions}}
\label{app:proof:shared_contrast_directions}

\begin{proof}
For any unit vector \(v\),
\[
v^\top \Delta G_y v
=
v^\top \overline G^{(y)}_{A,\psi,P} v
-
v^\top \overline G^{(y)}_{B,\psi,P} v.
\]
For a model \(m\), with \(h_m=f_m\circ\psi\),
\[
v^\top G^{(y)}_{m,\psi,P} v
=
\E_{s\sim\nu_y}
\!\left[
\|J_{h_m}(s)Pv\|_2^2
\right].
\]
Thus \(v^\top \Delta G_y v\) is the between-group contrast in expected local quadratic sensitivity along the shared coordinate-family direction \(v\).

The one-direction statements follow from the Rayleigh--Ritz variational characterization of eigenvalues \citep[Ch.~4]{horn2012matrix}:
\[
\max_{\|v\|_2=1} v^\top \Delta G_y v = \lambda_{\max}(\Delta G_y),
\qquad
\min_{\|v\|_2=1} v^\top \Delta G_y v = \lambda_{\min}(\Delta G_y).
\]
The \(r\)-dimensional statements follow from Ky Fan's variational principle \citep[Ch.~4]{horn2012matrix}:
\[
\max_{U^\top U=I_r}\Tr(U^\top \Delta G_y U)
=
\sum_{i=1}^r \lambda_i^\downarrow(\Delta G_y),
\]
with the minimum obtained by the \(r\) smallest eigenvalues. The same argument applies to \(\Delta\widehat G_y\).
\end{proof}

Exactly the same proposition applies if \(\Delta G_y\) is replaced by the shape-only contrast
\[
\Delta \widehat G_y
\coloneqq
\overline{\widehat G}^{(y)}_{A,P_K}-\overline{\widehat G}^{(y)}_{B,P_K},
\]
which isolates between-group differences in directional allocation independent of overall gain.

\section{The recurrence abstraction}
\label{app:recurrence}

The main text treats a representation as a differentiable map \(f:\R^d\to\R^m\). Recurrent systems fit the same abstraction by taking \(f\) to be either the unrolled state map at a fixed readout time or, when the dynamics converge, the equilibrium map.

\subsection{From recurrence to a representation map}

Consider a recurrent state \(r(t)\in\R^N\) driven by a constant stimulus \(x\in\R^d\) through an input drive \(u(x)\in\R^N\), with dynamics
\begin{equation}
\tau\frac{dr}{dt}
=
-r(t)+\sigma\!\big(Wr(t)+u(x)\big),
\label{eq:rnn_dynamics}
\end{equation}
where \(W\in\R^{N\times N}\) and \(\sigma\) is applied elementwise. Two representation maps are immediate:
(i) a finite-time map \(f_t(x)\coloneqq r(t;x)\), and
(ii) when \eqref{eq:rnn_dynamics} converges, an equilibrium map \(f(x)\coloneqq r^*(x)\), where \(r^*(x)\) satisfies
\begin{equation}
r^*(x)=\sigma\!\big(Wr^*(x)+u(x)\big).
\label{eq:fixed_point}
\end{equation}

\subsection{Geometric sensitivity via Jacobians}

In the unrolled case, \(J_{f_t}(x)\) is obtained by differentiating through time. In the equilibrium case, \(J_f(x)\) follows from implicit differentiation. Define
\[
F(r,x)\coloneqq r-\sigma(Wr+u(x)).
\]
At equilibrium, \(F(r^*(x),x)=0\). Differentiating with respect to \(x\) gives
\[
\frac{\partial F}{\partial r}\Big|_{(r^*,x)}\frac{dr^*}{dx}
+
\frac{\partial F}{\partial x}\Big|_{(r^*,x)}=0.
\]
Let
\[
D(x)\coloneqq \mathrm{diag}\!\big(\sigma'(Wr^*(x)+u(x))\big).
\]
Then
\[
\frac{\partial F}{\partial r}=\Id-D(x)W,
\qquad
\frac{\partial F}{\partial x}=-D(x)\frac{du}{dx}.
\]
Assuming \(\Id-D(x)W\) is invertible, the implicit function theorem yields
\begin{equation}
J_f(x)
=
\frac{dr^*}{dx}
=
\big(\Id-D(x)W\big)^{-1}D(x)\,\frac{du}{dx}.
\label{eq:implicit_jacobian}
\end{equation}

\subsection{Compatibility with the main-text geometry}

Once a recurrent network is cast as a differentiable map \(f\), the main-text definitions apply verbatim:
\[
M_f(x)=J_f(x)^\top J_f(x),
\qquad
G_f=\E_x[J_f(x)^\top J_f(x)].
\]
In this sense, pulling back to input space is not a change of object, but a coordinate description of local sensitivity of the same representation map.

\section{Additional experimental details}
\label{app:exp_setup}

\subsection{Model banks}
\label{app:model_banks}

We use three model banks, chosen to probe different kinds of representational variation.

\paragraph{Tiny10 bank.}
We trained 10 independent Tiny10 CNNs on CIFAR-10\cite{krizhevsky2009learning} dataset. This bank is intentionally small and homogeneous: all models share the same architecture and dataset, so differences arise only from optimization and random initialization. We use this bank primarily for sanity-check, especially layer-identity recovery and other analyses where a controlled setting is desirable before moving to larger or more heterogeneous models.

\paragraph{Standard-training banks.}
To study architectural variation under standard supervised training, we trained ResNet-18\cite{he2016deep} and a small vision transformer (ViT)\cite{dosovitskiy2020image} on CIFAR-10, using 10 seeds for each architecture--dataset pair. These banks are used for comparisons in which the dataset and broad supervised-training setting are fixed, while the compared model families use architecture-appropriate standard training recipes. The resulting variation therefore reflects trained model family, including architecture and its standard optimization recipe, together with seed-to-seed variability.

\paragraph{Adversarially trained banks.}
To study how local sensitivity geometry changes under different robustness objectives while holding architecture fixed, we trained ResNet-18 models on CIFAR-10 datasets with three adversarial-training procedures: PGD adversarial training, TRADES, and MART. For each dataset--method pair we trained 5 seeds. These banks are used for training-regime comparisons, where the architecture is fixed and the main source of variation is the learning objective.

Thus, taken together, the full suite lets us study local geometric sensitivity in three complementary regimes including within-architecture seed variability, across-architecture variation under standard training, and across-training-regime variation under fixed architecture.

\begin{table}[t]
\centering
\caption{Model banks used in the experiments.}
\label{tab:model_banks}
\begin{tabular}{llllc}
\toprule
Bank & Dataset & Model family & Training regime & Seeds \\
\midrule
Tiny10 & CIFAR-10 & Tiny10 CNN & standard training & 10 \\
Standard & CIFAR-10 & ResNet-18, small ViT & standard training & 10 each \\
Robust & CIFAR-10 & ResNet-18 & PGD / TRADES / MART & 5 each \\
\bottomrule
\end{tabular}
\end{table}

\subsection{Training details}
\label{app:training_details}

All models were trained from scratch with independent random seeds. Table~\ref{tab:training_hparams} summarizes the optimizer and schedule settings used for the model banks. Checkpoints were selected consistently within each bank and then reused unchanged for all activation-based and sensitivity-based analyses.

\begin{table}[t]
\centering
\scriptsize
\setlength{\tabcolsep}{3pt}
\renewcommand{\arraystretch}{1.10}
\caption{\textbf{Training hyperparameters for computational model banks.} The standard ResNet-18 and small-ViT banks use architecture-appropriate supervised training recipes; robust-training comparisons hold architecture fixed and vary only the training objective.}
\label{tab:training_hparams}
\resizebox{\linewidth}{!}{%
\begin{tabular}{@{}llllrrrrrrl@{}}
\toprule
Bank & Dataset & Model & Optimizer & Epochs & Batch & LR & WD & Warmup & Label smooth. & Schedule / selection \\
\midrule
Tiny10 & CIFAR-10 & Tiny10 CNN & AdamW & 100 & 128 & \(10^{-3}\) & \(10^{-4}\) & 5 & 0.00 & warmup + cosine; consistent saved checkpoint \\
Standard & CIFAR-10 & ResNet-18 & SGD/Nesterov & 200 & 512 & 0.10 & \(5{\times}10^{-4}\) & 5 & 0.10 & warmup + cosine to \(10^{-2}\) LR; early stop after min 60 \\
Standard & CIFAR-10 & small ViT & AdamW & 300 & 512 & \(3{\times}10^{-4}\) & \(5{\times}10^{-2}\) & 20 & 0.10 & warmup + cosine to \(10^{-2}\) LR; early stop after min 100 \\
Robust & CIFAR-10 & ResNet-18 PGD & SGD/momentum & 110 & 256 & 0.10 & \(5{\times}10^{-4}\) & -- & -- & step decay at 55/75/90 epochs; best robust checkpoint \\
Robust & CIFAR-10 & ResNet-18 TRADES & SGD/momentum & 110 & 256 & 0.10 & \(5{\times}10^{-4}\) & -- & -- & step decay at 55/75/90 epochs; best robust checkpoint \\
Robust & CIFAR-10 & ResNet-18 MART & SGD/momentum & 110 & 256 & 0.10 & \(5{\times}10^{-4}\) & -- & -- & step decay at 55/75/90 epochs; best robust checkpoint \\
\bottomrule
\end{tabular}%
}
\end{table}

\paragraph{Standard-model training.}
For the standard ResNet-18 models, we use SGD with Nesterov momentum \(0.9\), decoupled weight decay on non-bias/non-normalization parameters, label smoothing \(0.10\), and a warmup-cosine learning-rate schedule. For small ViT, we use AdamW with gradient clipping at global norm \(1.0\), label smoothing \(0.10\), and the same warmup-cosine schedule form. The cosine schedule decays to \(10^{-2}\) times the base learning rate. CIFAR training uses random crop/flip augmentation and dataset-specific channel normalization. Early stopping is allowed only after the minimum epoch shown in the training code and uses held-out evaluation accuracy.

\paragraph{Tiny10 training.}
Tiny10 models use AdamW with batch size \(128\), \(100\) epochs, base learning rate \(10^{-3}\), weight decay \(10^{-4}\), and a \(5\)-epoch linear warmup followed by cosine decay to \(1\%\) of the base learning rate. CIFAR-10 images are scaled to \([-1,1]\). Training augmentation uses reflection padding by \(4\) pixels, random crop back to \(32\times32\), and random horizontal flip.

\paragraph{Robust ResNet-18 training.}
Robust models use \(\ell_\infty\) adversarial training with radius \(\epsilon=8/255\), step size \(2/255\), \(10\) inner-loop PGD steps for training, and \(20\) PGD steps for robust evaluation. PGD models minimize adversarial cross-entropy. TRADES uses KL-based adversarial examples and \(\beta=6.0\). MART uses the MART loss with \(\beta=5.0\). All robust models use SGD with momentum \(0.9\), weight decay \(5\times10^{-4}\), batch size \(256\), \(110\) epochs, and learning-rate drops by a factor \(0.1\) at epochs \(55\), \(75\), and \(90\). Robust checkpoints are selected by held-out robust accuracy.

\subsection{Layer definitions and extracted representations}
\label{app:layer_defs}

For each trained network, we treat a \emph{layer representation} as the input-to-layer map \(\Phi_n\) introduced in Section~\ref{sec:theory}. In practice, this means that for a given checkpoint we extract both activations \(\Phi_n(x)\) and local sensitivity information through Jacobian-vector products of the same map.

\paragraph{Tiny10 CNNs.}
For Tiny10, candidate layers correspond to the post-nonlinearity outputs of the convolutional stages and the corresponding final hidden layer representation used before the classifier. This bank is used when we want a relatively fine-grained but still manageable sequence of convolutional representations.

\paragraph{ResNet-18.}
For ResNet-18, we use the stem representation, the outputs of the residual blocks, and the pooled pre-logit representation. This yields a sequence of intermediate maps that naturally tracks the model from early local feature extraction to the final global representation.

\paragraph{Small ViT.}
For the small ViT, we use the patch-embedding representation, the token representation before the encoder stack, the outputs of successive transformer blocks, and the pre-logit class-token representation. This gives an architecture-matched sequence of intermediate transformer representations without requiring a forced one-to-one correspondence with convolutional layers.

Throughout, whenever we compare two layers, the compared objects are the corresponding saved representation maps and their associated dataset-level sensitivity summaries, rather than individual activation vectors from a single image.

\subsection{Evaluation protocol and aggregation}
\label{app:eval_protocol}

Unless stated otherwise, comparisons are made between independently trained models. For within-bank analyses, we aggregate over seed pairs so that reported values reflect stable trends rather than idiosyncrasies of a single training run. For layer-identification experiments, each candidate layer in one model is compared against the candidate layers in another model, and prediction is counted as correct when the matched layer index agrees with the true layer identity. For cross-model disagreement analyses, distances are computed from the relevant representation layers and then related to independently measured perturbation-task disagreement across the same model pairs.

This protocol is used to separate the following questions. Whether the measure is stable across random initializations, whether it captures architecture-level differences beyond seed noise, and whether it is sensitive to systematic changes in the training objective such as adversarial training.

\subsection{Compute resources}
\label{app:compute_resources}

The computational-model experiments were run in Google Colab TPU v6e High-RAM environments using Python/JAX workflows. The Tiny10 layer-matching experiments, activation baselines, class-conditional probe analyses, and robust-training probe analyses can be reproduced on a comparable modern accelerator runtime. The main memory requirements come from storing activations, projected Jacobian-vector-product summaries, and pairwise layer/model similarity matrices. The sweeps over model pairs, layers, task-family dimensions \(K\), probe counts, and training-regime pairs are independently parallelizable, and in our runs were executed as separate accelerator jobs rather than as one monolithic run.

The Allen static-gratings analysis was run locally on a CPU workstation/server. The main requirements are memory for matched-count response matrices, condition means, pooled covariances, and per-experiment \(3\times3\) Fisher/pullback summaries. A workstation with 16--32 CPU cores and 64--128GB RAM is sufficient for reproducing the matched-count area and VISp-depth analyses; more cores primarily speed up independent experiment, subsample, and family-restriction loops. The biological pipeline is CPU-parallelizable across experiments and matched subsamples and does not require a GPU.

\section{Additional Experiments.}

\subsection{Layer matching sanity check and task-family dependence}
\label{app:layer_matching}

This appendix provides experimental details and robustness analyses for the layer-matching sanity check reported in Figure~\ref{fig:layer_identification_main}. The goal of this experiment is not to establish a new benchmark in its own right, but to verify that the proposed geometry produces a sensible layer-identity signal in the setting introduced by \citet{kornblith2019cka}, corresponding layers in independently trained networks should be most similar to one another. In our framework, this sanity check is performed relative to explicit local perturbation task families.

\paragraph{Experimental setup.}
We train $10$ independently initialized Tiny10 convolutional networks on CIFAR-10 and compare the eight intermediate ReLU layers across all unordered model pairs. For each pair of models $(a,b)$ and each task-family dimension $K$, we form a layer-similarity matrix
\[
S^{(a,b)}_{K}\in\mathbb{R}^{8\times 8},
\qquad
S^{(a,b)}_{K}(i,j)
=
\exp\!\left(
-\frac{d_{\mathrm{AIRM}}\!\big(\widetilde G_{\Phi_i^{(a)},P_K},\widetilde G_{\Phi_j^{(b)},P_K}\big)}{\sqrt{K}}
\right),
\]
where $\Phi_i^{(a)}$ denotes the input-to-layer map for layer $i$ in model $a$, $G_{\Phi_i^{(a)},P_K}$ is the corresponding subspace-restricted expected pullback metric, and $\widetilde G$ denotes the trace-scaled SPD lift used throughout the paper. For comparison, we also compute layer-similarity matrices using linear CKA and RBF CKA on the same input-to-layer maps.

We report the average similarity matrix
\[
\bar S_K
\coloneqq
\frac{1}{|\mathcal P|}\sum_{(a,b)\in\mathcal P} S^{(a,b)}_K,
\]
where $\mathcal P$ is the set of unordered model pairs. The displayed heatmaps in both the main text and appendix are these pair-averaged matrices.

\paragraph{Task families.}
The main-text sanity check uses a \emph{random-basis} family. We sample a single parent orthonormal matrix
\[
P^{\mathrm{rand}}_{K_{\max}}\in\mathbb{R}^{d\times K_{\max}},
\qquad
(P^{\mathrm{rand}}_{K_{\max}})^\top P^{\mathrm{rand}}_{K_{\max}}=I,
\]
and define nested restrictions
\[
P^{\mathrm{rand}}_{K}
=
P^{\mathrm{rand}}_{K_{\max}}[:,1{:}K].
\]
Perturbations are then drawn as
\[
\delta = P^{\mathrm{rand}}_{K} z,
\qquad
z\sim\mathcal N(0,I_K).
\]
Thus the admissible perturbations lie in a shared random $K$-dimensional subspace, while directions inside that subspace are weighted isotropically. 

As a robustness analysis, we also consider a \emph{PCA-basis} family. Let
\[
\mu \coloneqq \E_{x\sim\mathcal D}[x],
\qquad
\Sigma_{\mathcal D}
\coloneqq
\E_{x\sim\mathcal D}\big[(x-\mu)(x-\mu)^\top\big].
\]
Write the eigendecomposition of the stimulus covariance as
\[
\Sigma_{\mathcal D}=U\Lambda U^\top,
\qquad
\Lambda=\mathrm{diag}(\lambda_1,\dots,\lambda_d),
\qquad
\lambda_1\ge \cdots \ge \lambda_d\ge 0,
\]
with columns of $U$ given by principal directions $u_1,\dots,u_d$. For dimension $K$, define
\[
P^{\mathrm{PCA}}_{K}
\coloneqq
[u_1,\dots,u_K]\in\mathbb{R}^{d\times K},
\qquad
(P^{\mathrm{PCA}}_{K})^\top P^{\mathrm{PCA}}_{K}=I_K.
\]
We again take
\[
\delta = P^{\mathrm{PCA}}_{K} z,
\qquad
z\sim\mathcal N(0,I_K).
\]
This defines a data-adapted but within-subspace isotropic family, where the support of the task family is determined by the dominant principal directions of the dataset, but all directions inside that span are weighted equally. Equivalently, this is a PCA-basis second-moment local task family with $C_z=I_K$.

For either choice of $P_K$, the corresponding expected pullback metric is
\[
G_{f,P_K}
\coloneqq
\E_{x\sim\mathcal D}\!\left[(J_f(x)P_K)^\top (J_f(x)P_K)\right]\in\mathbb{R}^{K\times K},
\]
and expected local discriminability over the family is
\[
\E[d'^2] = \Tr(C_z\, G_{f,P_K}).
\]
Under the isotropic choice $C_z=I_K$, $G_{f,P_K}$ remains the minimal dataset-level summary that preserves expected performance over the chosen subspace-restricted family.

\paragraph{Evaluation metrics.}
The main-text and appendix figures use several summary statistics derived from the pairwise similarity matrices.

First, layer-identification accuracy is defined by checking whether the correct corresponding layer is the top match in each row and column:
\[
\mathrm{Acc}_K
=
\frac{1}{2L|\mathcal P|}
\sum_{(a,b)\in\mathcal P}
\left[
\sum_{i=1}^{L}\mathbf 1\!\left\{\arg\max_j S^{(a,b)}_K(i,j)=i\right\}
+
\sum_{j=1}^{L}\mathbf 1\!\left\{\arg\max_i S^{(a,b)}_K(i,j)=j\right\}
\right]\times 100,
\]
where $L=8$ is the number of compared layers.

Second, the \emph{similarity decay by layer distance} averages similarity over entries with a fixed absolute layer offset:
\[
D_K(\Delta)
\coloneqq
\frac{1}{|\{(i,j): |i-j|=\Delta\}|}
\sum_{|i-j|=\Delta} \bar S_K(i,j).
\]
This is the quantity plotted in panel A of Figure~\ref{fig:layer_identification_main}.
\begin{figure*}[t]
    \centering
    \includegraphics[width=\textwidth]{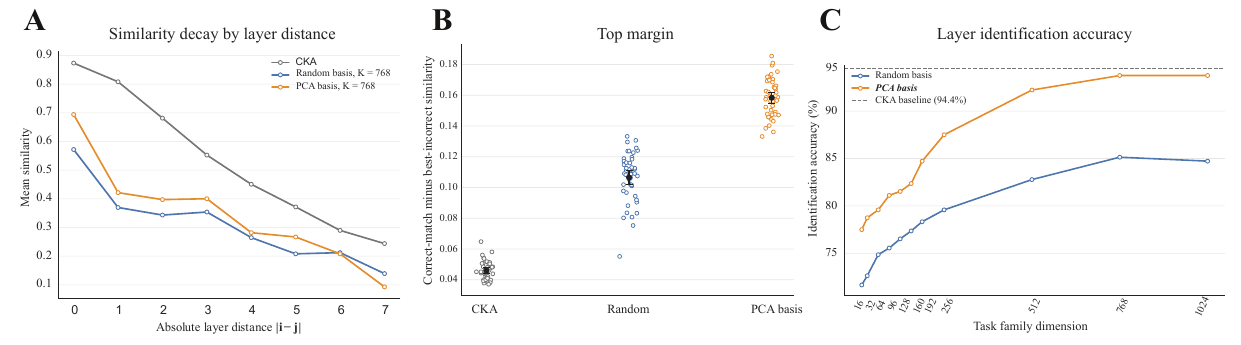}
    \caption{
    \textbf{Layer-matching diagnostics.}
    (A) Mean similarity as a function of absolute layer distance for CKA and S-RAS under random-basis and PCA-basis families at \(K=768\).
    (B) Top-margin statistic, defined as the correct-match similarity minus the strongest incorrect-match similarity.
    (C) Layer-identification accuracy as a function of task-family dimension \(K\), with linear CKA shown as a dashed reference.
    }
    \label{fig:layer_identification_main}
\end{figure*}

\begin{figure}[t]
    \centering
    \includegraphics[width=\linewidth]{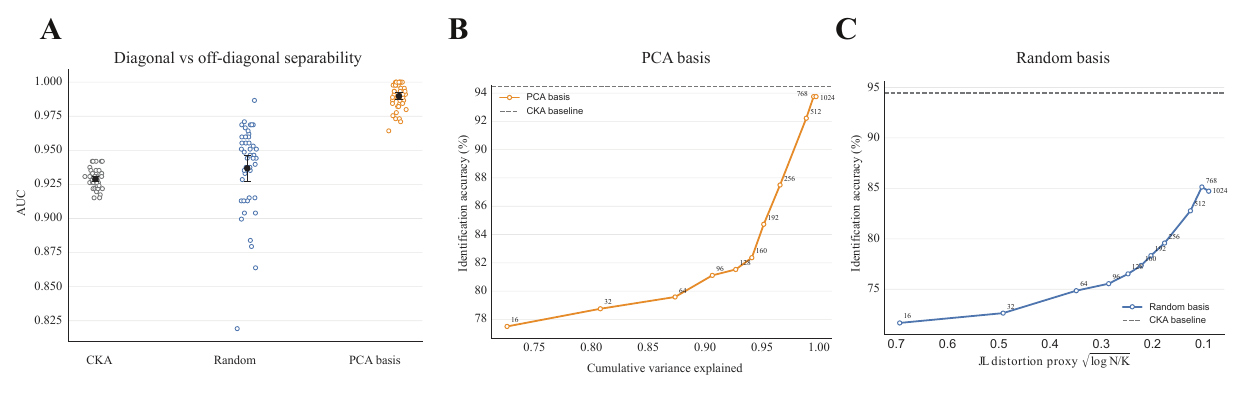}
    \caption{
    \textbf{Task-family dependence of layer-identification performance and off-diagonal separability.}
    \textbf{A}, diagonal-versus-off-diagonal separability, measured by AUC, for linear CKA, S-RAS under the random-basis family, and S-RAS under the PCA-basis family.
    \textbf{B}, layer-identification accuracy for the PCA-basis family as a function of cumulative variance explained by the retained principal directions.
    \textbf{C}, layer-identification accuracy for the random-basis family as a function of the Johnson--Lindenstrauss distortion proxy $\sqrt{\log N / K}$.
    Together, these panels show that the task-family choice changes the geometric axis along which performance is organized. In the PCA-basis family, performance tracks the amount of dataset variance admitted into the family, whereas in the random-basis family, performance improves as the expected projection distortion decreases.
    }
    \label{fig:layer_identification_appendix_family_dependence}
\end{figure}
Third, the \emph{top-1 margin} is defined for each row and column as the similarity of the correct layer match minus the strongest incorrect match, then averaged across rows and columns. This is the quantity plotted in panel B of Figure~\ref{fig:layer_identification_main}.

Fourth, the appendix uses a threshold-free summary of diagonal-versus-off-diagonal separability. For each pairwise similarity matrix, diagonal entries are treated as positives and off-diagonal entries as negatives, and the resulting area under the ROC curve (AUC) is computed. Larger values indicate stronger separation between correct and incorrect matches.

Finally, for the family-dependence plots in Figure~\ref{fig:layer_identification_appendix_family_dependence}, we summarize the horizontal axis differently for the two families. In the PCA-basis family, we report cumulative explained variance of the retained principal directions. In the random-basis family, we report the Johnson--Lindenstrauss distortion proxy
\[
\epsilon_{\mathrm{JL}}(K)\propto \sqrt{\frac{\log N}{K}},
\]
with $N$ the number of sampled images, as a simple measure of how projection distortion should decrease as $K$ grows.

\begin{figure}[t]
    \centering
    \includegraphics[width=0.98\linewidth]{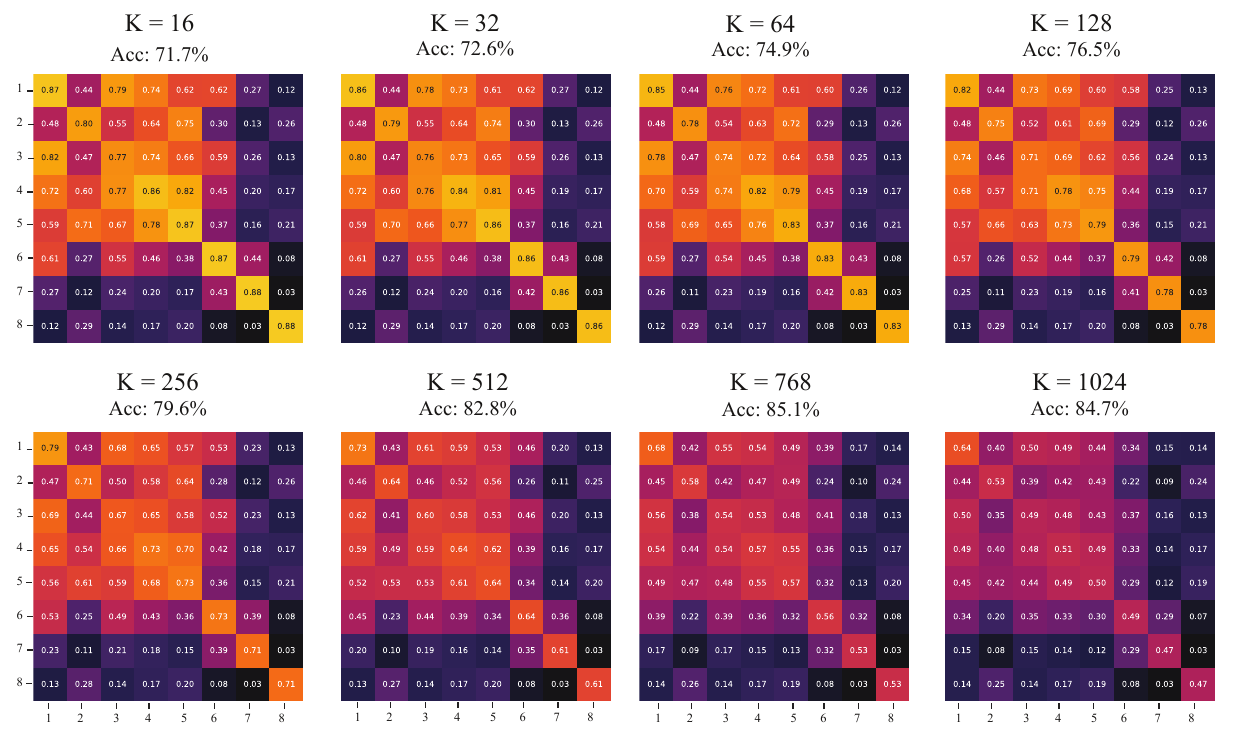}
    \vspace{-0.7em}
    \caption{
    \textbf{Random-basis task-family sweep for S-RAS layer identification.}
    Each panel shows the pair-averaged S-RAS similarity matrix for a random-basis subspace-restricted local perturbation task family of dimension $K$, together with the corresponding layer-identification accuracy.
    As the task-family dimension increases from $K=16$ to $K=1024$, accuracy rises from $71.7\%$ to $84.7\%$, with the highest value reached at $K=768$ and similar performance at $K=1024$.
    Across the sweep, the similarity structure becomes increasingly concentrated near the diagonal while distant off-diagonal entries are more strongly suppressed.
    These panels should be read not as increasingly accurate random \emph{sketches} of a fixed object, but as nested restrictions of the random-basis local perturbation task family used to probe dataset-averaged sensitivity geometry.
    }
    \label{fig:heatmaps_random}
    \vspace{-0.9em}
\end{figure}

\begin{figure}[t]
    \centering
    \includegraphics[width=0.98\linewidth]{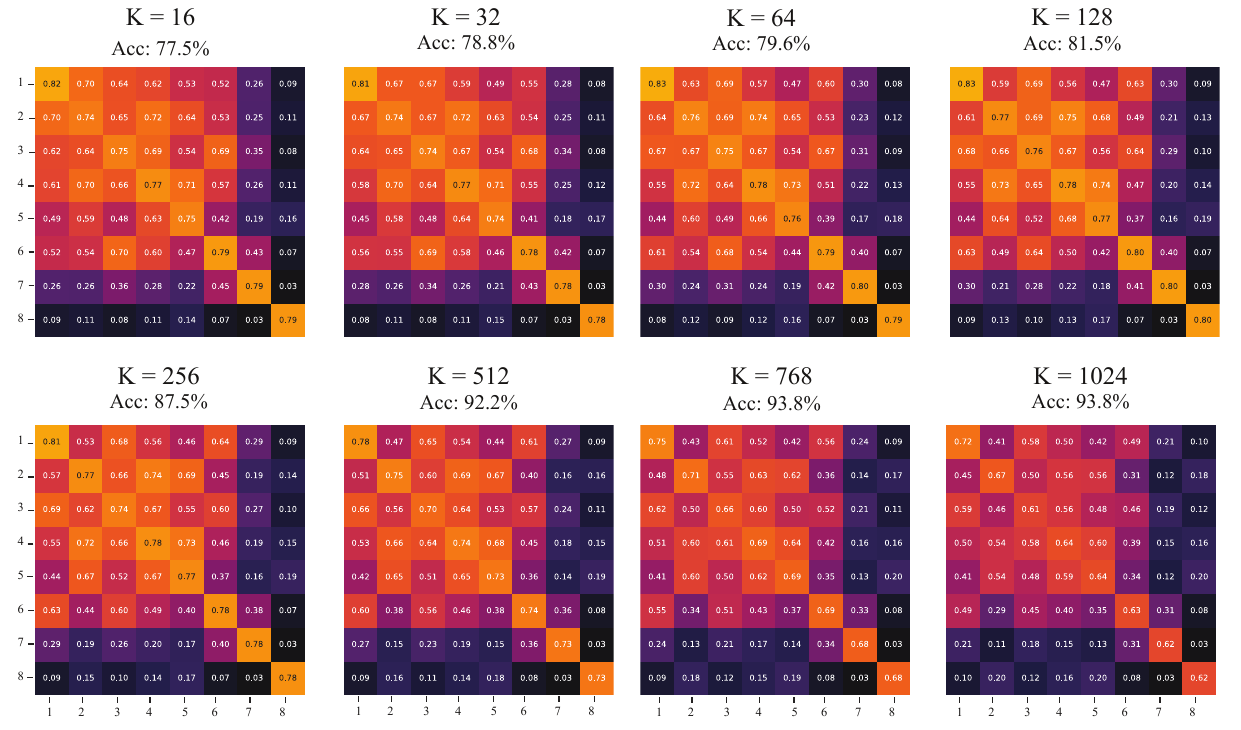}
    \vspace{-0.7em}
    \caption{
    \textbf{PCA-basis task-family sweep for S-RAS layer identification.}
    Each panel shows the pair-averaged S-RAS similarity matrix for a PCA-basis subspace-restricted local perturbation task family of dimension $K$, together with the corresponding layer-identification accuracy.
    Here the admissible perturbations are restricted to the top-$K$ principal directions of the dataset, but weighted isotropically within that span.
    As $K$ increases from $16$ to $1024$, identification accuracy rises from $77.5\%$ to $93.8\%$, and the similarity structure becomes more tightly concentrated around the diagonal than in the random-basis family.
    This indicates that when the local task family is aligned with dominant directions of dataset variation, corresponding layers are recovered even more reliably, while distant incorrect matches remain strongly suppressed.
    }
    \label{fig:heatmaps_pca}
    \vspace{-0.9em}
\end{figure}

\paragraph{Compute and memory profile.}
Table~\ref{tab:layer_identification} uses \(10\) Tiny10 models, \(8\) layers, \(45\) unordered model pairs, and the same fixed \(n=1024\) CIFAR-10 images for reported methods. The trace-scaled SPD lift uses \(\epsreg=10^{-4}\). Dataset-level \sras\ stores one \(K\times K\) expected operator per model, layer, and family, so the dominant cache for one \(K=768\) operator is \(768^2\) float32 entries. By contrast, pointwise AIRM and the MSA-style control store imagewise projected metrics of size \(n\times K\times K\), making their memory scale linearly with the number of images in addition to quadratically in \(K\). This is why the pointwise local-geometry controls are reported up to \(K=256\), whereas dataset-level \sras\ is reported at \(K=768\). Activation baselines require only forward activations and sample-space similarity computations, and are therefore not compute-matched to the JVP-based local-geometry methods. The comparison in Table~\ref{tab:layer_identification} is therefore sample-matched and checkpoint-matched, but not compute-matched.

\begin{table}[t]
\centering
\scriptsize
\setlength{\tabcolsep}{3pt}
\renewcommand{\arraystretch}{1.08}
\caption{\textbf{Layer-matching compute and memory scaling.} All methods in Table~\ref{tab:layer_identification} use the same \(n=1024\) images and \(45\) unordered model pairs. Memory entries describe the dominant cached array type, excluding model parameters and framework workspace.}
\label{tab:layer_matching_compute}
\begin{tabular}{@{}p{0.20\linewidth}p{0.31\linewidth}p{0.39\linewidth}@{}}
\toprule
Method & Dominant cached object & Scaling and practical consequence \\
\midrule
\sras\ full / shape
& Dataset-level \(G\in\mathbb R^{K\times K}\) per model, layer, and family.
& Memory scales as \(O(K^2)\) per summary after streaming over images. This permits the \(K=768\) high-\(K\) reference. Shape-only \sras\ reuses the same summaries after trace normalization. \\
\midrule
pw-AIRM
& Imagewise projected metrics \(G_x\in\mathbb R^{n\times K\times K}\).
& Memory scales as \(O(nK^2)\), and pairwise comparisons average \(n\) local AIRM distances. This motivates the \(K=256\) cap for pointwise controls. \\
\midrule
MSA-style spectral ratio
& Same imagewise projected metrics as pw-AIRM.
& Uses the same \(O(nK^2)\) cache, then averages a local spectral-ratio distance over images. This is a down-projected control, not a full ambient-dimensional MSA benchmark. \\
\midrule
CKA / Procrustes / CCA
& Forward activations, Gram matrices, or sample-space SVD factors.
& No JVPs or imagewise local metric cache are required. These baselines are sample-matched but not compute-matched to local-geometry methods. \\
\bottomrule
\end{tabular}
\end{table}

\paragraph{Results.}

Figure~\ref{fig:layer_identification_main} provides the diagnostic summary for the layer-matching experiment. Panel A shows that S-RAS similarities decay more sharply with absolute layer distance than CKA, indicating stronger suppression of distant off-diagonal matches. Panel B shows the corresponding top-margin statistic: the correct match is more separated from the strongest incorrect match under the projected sensitivity families, especially for the PCA-basis family. Panel C reports the \(K\)-sweep, showing that layer-identification accuracy improves with task-family dimension and that the PCA-basis family approaches the CKA baseline at high \(K\).

The full random-basis and PCA-basis heatmap sweeps are shown in Figures~\ref{fig:heatmaps_random} and~\ref{fig:heatmaps_pca}. These supplementary results show that the correspondence signal strengthens with \(K\), but not perfectly monotonically, reflecting both increased retained dimension and the particular directions included in each perturbation family. The family-dependence analyses in Figure~\ref{fig:layer_identification_appendix_family_dependence} further show that random and PCA-basis families organize the comparison differently: random-family performance tracks reduced projection distortion, whereas PCA-basis performance tracks the amount of dataset variance admitted into the family. Appendix~\ref{app:sr_diagnosis} gives diagnostic context for the down-projected spectral-ratio baseline, including its invariance to conformal gain changes and its dependence only on the extreme generalized eigenvalues.

\subsection{Held-out class-conditional diagnostic probes}
\label{app:diagnostic_probes}

This appendix gives the exact construction used in Section~\ref{sec:diagnostic_probes}. The goal is to test whether the class-conditional projected expected pullback metric can be externalized into a small finite probe set that generalizes to held-out models and held-out images.
\paragraph{Class-conditional summaries.}
For each image \(x\), let \(\psi_x(u)\), \(u\in\mathbb R^q\), denote the smooth-deformation parameterization and define
\[
h_{m,x}(u)=f_m(\psi_x(u)).
\]
For each model \(m\), family dimension \(K\), and class \(y\), we estimate
\[
G^{(y)}_{m,P_K}
=
\E_{x\sim \mathcal D_y}\!\left[
(J_{h_{m,x}}(0)P_K)^\top
(J_{h_{m,x}}(0)P_K)
\right]
\in\mathbb S_+^K .
\]

Here \(q\) is the dimension of the learned deformation-coordinate system, and \(P_K\in\mathbb R^{q\times K}\) is the \(K\)-dimensional task-family subspace inside that system. It is taken as a nested restriction of a single maximal basis so that comparisons across \(K\) remain within the same ambient learned family.

The perturbation family is built from a broad boundary-safe smooth deformation bank. Concretely, we begin with a \(6\times 6\) bank of scalar sine modes on the image lattice and form separate horizontal and vertical displacement fields, yielding \(72\) base flows in total. For each image, base tangents are estimated by centered finite differences with coefficient \(0.5\) pixels RMS. We then learn a nested family basis from the empirical covariance of these tangents on a separate family-learning split of \(500\) training images. All reported \(K\in\{8,16,32,64\}\) families are nested restrictions of this single learned \(K_{\max}=64\) family, so changing \(K\) changes family dimension without changing the ambient learned basis.

\paragraph{Discovery/test split.}
Image pools are separated by role. The deformation family basis is learned from \(500\) CIFAR-10 training images. Class-conditional sensitivity contrasts are estimated on a disjoint \(2000\)-image CIFAR-10 training split. Probe evaluation uses only CIFAR-10 test images: from a deterministic candidate pool of \(256\) test images, we retain the first \(128\) that are clean-correct for at least \(12\) models in total and at least \(4\) models from each architecture group across the full bank.

We use \(10\) balanced model splits. In each split, the two architecture groups are partitioned into \(5\) discovery and \(5\) held-out models per group. Probe directions are derived using only the discovery models and the training-image contrast summaries. Evaluation is then performed only on held-out models and the retained held-out test images. During held-out evaluation, an image contributes only if at least \(2\) held-out models from each group remain valid on that image. Thus the finite probes are derived from training images and discovery models, whereas the reported separations are evaluated on held-out test images and held-out models.

\paragraph{Probe construction.}
For a class \(y\), let
\[
G^{(y)}_{A,P_K}
\qquad\text{and}\qquad
G^{(y)}_{B,P_K}
\]
denote the discovery-set class-conditional averages for the two model groups. We define the contrast operator
\[
\Delta G_y
\coloneqq
G^{(y)}_{A,P_K}-G^{(y)}_{B,P_K}.
\]
Let
\[
\Delta G_y = U\Lambda U^\top
\]
be its eigendecomposition. The top positive eigenvectors define probes favoring group \(A\), while the most negative eigenvectors define probes favoring group \(B\). We use \(P\) probes per side and instantiate each probe direction as a finite perturbation in the learned family.

\paragraph{Probe score.}
Each selected eigen-direction is evaluated with both finite warp signs to reduce dependence on odd finite-amplitude effects. Let \(b\in\{+,-\}\) index the two contrast sides, corresponding to the positive and negative eigendirections of \(\Delta G_y\), let \(p\in\{1,\ldots,P\}\) index probes within a side, let \(\sigma\in\{-1,+1\}\) denote the finite warp sign, and let \(a\in\mathcal A=\{0.5,1.0\}\) denote RMS-flow amplitude. For model \(m\), image \(x\), and finite perturbation \(\delta_{b,p,\sigma,a}\), define
\[
\Delta M_m(x;\delta)
=
M_m(x)-M_m(x+\delta),
\]
where \(M_m\) is the true-class logit margin. We first symmetrize within each contrast side,
\[
R_m^{b}(x)
=
\frac{1}{P|\mathcal A|\,2}
\sum_{p=1}^{P}
\sum_{a\in\mathcal A}
\sum_{\sigma\in\{-1,+1\}}
\Delta M_m(x;\delta_{b,p,\sigma,a}),
\]
and then define the held-out probe score as
\[
s_m(x)=R_m^{+}(x)-R_m^{-}(x).
\]
Thus the finite perturbation response is averaged over warp signs and amplitudes within each side, while the final score is the contrast between positive- and negative-eigenvector probe sides.

\paragraph{Aggregation and uncertainty.}
For each split, class, probe type, family dimension, amplitude, and probe count, we average \(s_m(x)\) over valid held-out images and held-out models to obtain a split-class image-separation row. For formal statistical testing, we aggregate these rows over classes within each held-out model split and perform paired one-sided Wilcoxon tests across the \(10\) split-level paired observations. Split-class scatter plots are retained only as diagnostic visualizations of heterogeneity across classes. Unless otherwise stated, appendix summaries report means over the saved split-class rows; the formal \(p\)-values use the split-level aggregation.

\paragraph{Control probe sets.}
We compare the contrast-derived probes to three controls.
\begin{enumerate}
    \item \textbf{Random-contrast probes:} random directions sampled in the same family and ranked by the same quadratic contrast criterion.
    \item \textbf{Pooled-sensitivity probes:} directions drawn from high-sensitivity directions under \(G^{(y)}_{A,P_K}+G^{(y)}_{B,P_K}\) and then ranked by the same contrast score within that restricted candidate set.
    \item \textbf{Label-permutation null:} probes derived after randomly permuting discovery-set group labels, rebuilding the contrast from those permuted labels, and deriving probes from that null contrast.
\end{enumerate}

\begin{figure*}[t]
    \centering
    \includegraphics[width=\textwidth]{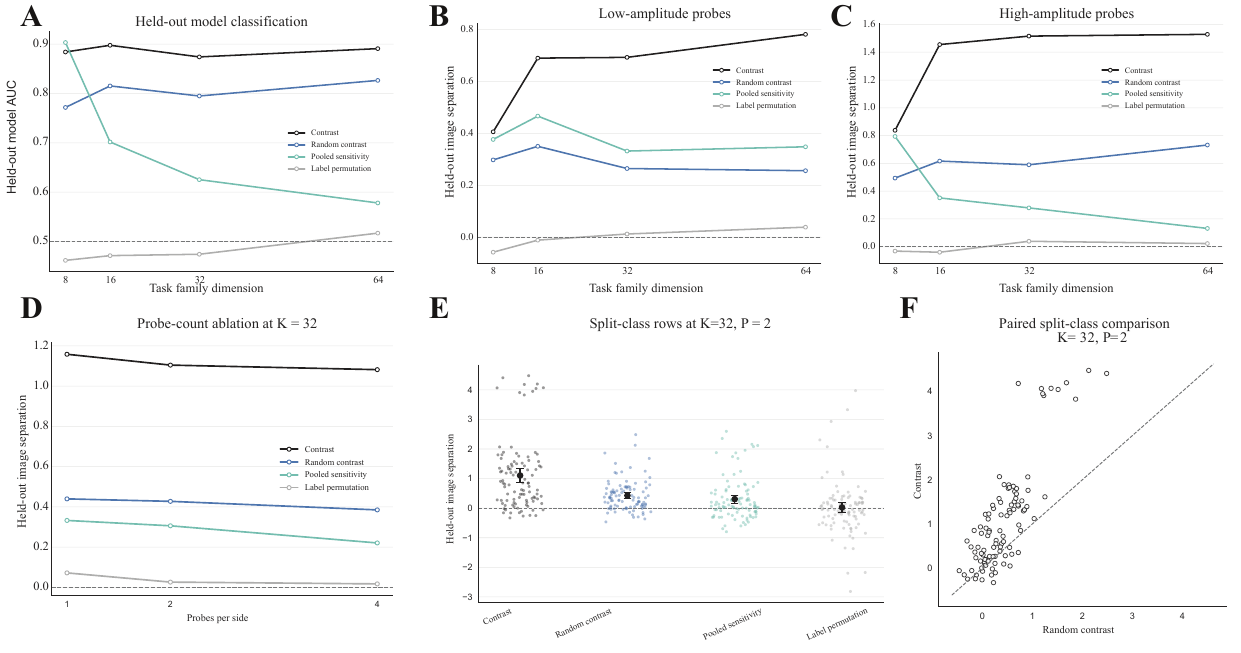}
    \caption{
    \textbf{Additional ablations for held-out class-conditional diagnostic probes.}
    \textbf{A}, held-out model classification AUC as a function of task-family dimension \(K\).
    \textbf{B}, low-amplitude version of the held-out image-separation analysis.
    \textbf{C}, high-amplitude version of the held-out image-separation analysis.
    \textbf{D}, probe-count ablation at the main setting (\(K=32\)).
    \textbf{E}, distribution of split-class rows at the main setting.
\textbf{F}, split-class diagnostic visualization against the random-contrast baseline.
These panels show that the held-out image-separation effect is robust across perturbation amplitude and probe count. Split-class rows are shown as a diagnostic visualization; formal statistical tests in the main result aggregate first to the held-out model-split level.
    }
    \label{fig:diagnostic_probes_appendix}
\end{figure*}

\paragraph{Main result.}
At the main setting (\(K=32\), two probes per side), contrast-derived probes achieve mean held-out image separation \(1.10\), compared with \(0.43\) for random-contrast probes, \(0.31\) for pooled-sensitivity probes, and \(0.03\) for the label-permutation null. Using the split-level test defined above, contrast-derived probes outperform each control, with one-sided paired Wilcoxon \(p=9.77\times10^{-4}\) for contrast versus random contrast, contrast versus pooled sensitivity, and contrast versus label permutation.

\paragraph{Held-out model classification.}
As a secondary diagnostic, we also ask whether the average probe score can classify held-out models by group. At the same setting, mean held-out model AUC is \(0.874\) for contrast-derived probes, compared with \(0.795\) for random-contrast probes, \(0.625\) for pooled-sensitivity probes, and \(0.474\) for the label-permutation null. Because this comparison is between ResNet18 and small ViT, we treat held-out model classification as supporting evidence rather than as the primary metric.

\paragraph{Ablations.}
We report results for nested family dimensions \(K\in\{8,16,32,64\}\), for probe counts \(P\in\{1,2,4\}\) per side, and for low- and high-amplitude perturbations separately. In the saved runs, the main held-out image-separation effect is already strong by \(K=16\), remains strong at \(K=32\), and changes only modestly at \(K=64\). At \(K=32\), mean held-out image separation for contrast-derived probes is \(1.16\), \(1.10\), and \(1.08\) for \(1\), \(2\), and \(4\) probes per side, respectively, showing that the main effect is also stable across probe-count choices. These auxiliary analyses support the interpretation that the main result is not tied to one fragile choice of family dimension, amplitude, or probe budget.

\subsection{Robust training as a controlled dissociation}
\label{app:robust_dissociation}

This appendix reuses the same learned-family, finite-probe, and held-out evaluation framework as Appendix~\ref{app:diagnostic_probes}, but now in a fixed-architecture bank of CIFAR-10 ResNet-18 models trained with four loss functions: standard empirical-risk minimization, PGD adversarial training, TRADES, and MART. The purpose of this experiment is to remove architecture as a confound and ask whether the learned probe battery can reveal differences in local evidence use that are attributable to training objective alone.

For each regime pair, we generate \(10\) matched discovery/test splits. Because each robust bank contains \(5\) models, each split uses \(3\) discovery and \(2\) held-out models per group. Unless otherwise noted, the main analysis uses the pre-logit representation, \(K\in\{8,16,32,64\}\), and two probes per side. Benchmark images are selected from a candidate pool of \(256\) test images, retaining the first \(128\) that are clean-correct for at least \(12\) models in total. For each regime pair and family size \(K\), we form a discovery-set contrast between the corresponding regime-level sensitivity summaries, derive contrast-favoring and contrast-opposing probe directions from that operator, and evaluate them on held-out models and held-out images. During held-out evaluation, an image contributes whenever at least one held-out model from each group is valid on that image.

\begin{figure*}[t]
    \centering
    \includegraphics[width=\textwidth]{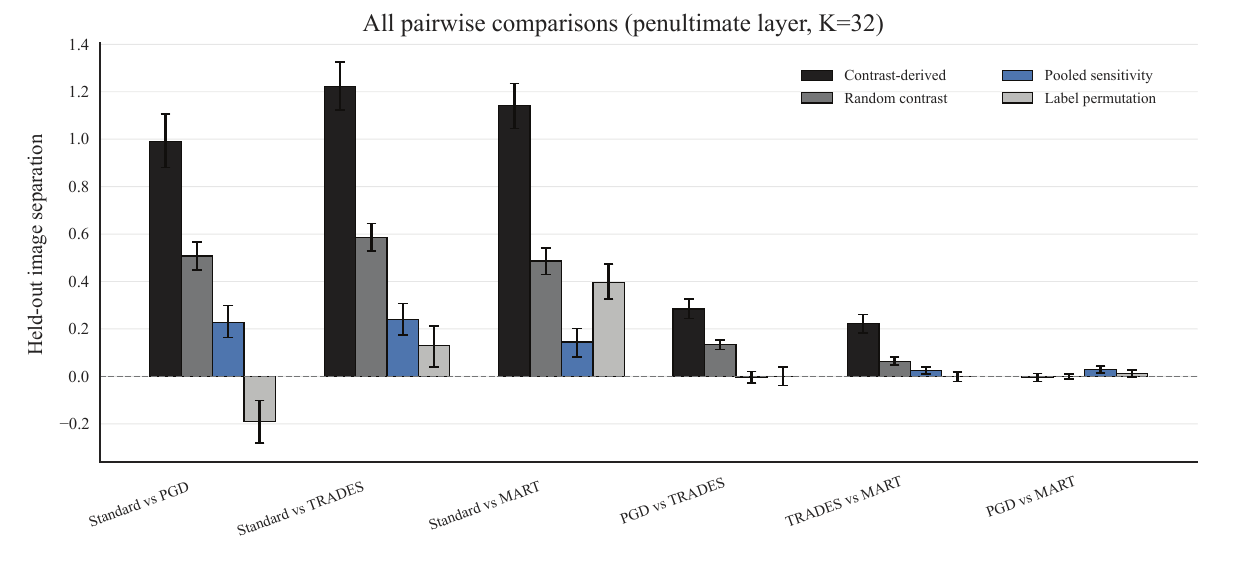}
    \caption{
    \textbf{All pairwise regime comparisons at the main setting.}
    Held-out image separation at \(K=32\) for all six pairwise comparisons among standard, PGD, TRADES, and MART-trained ResNet-18 models, comparing contrast-derived probes to the same three controls used in the main text.
    The three standard-versus-robust comparisons are strongest, TRADES versus MART is smaller but positive, and PGD versus MART shows little or no separation under the present family.
    }
    \label{fig:adv_std_all_appendix}
\end{figure*}

Figure~\ref{fig:adv_std_all_appendix} reports all six pairwise comparisons at \(K=32\). The three standard-versus-robust comparisons are strongest, with mean held-out image separation \(0.99\) (standard vs PGD), \(1.22\) (standard vs TRADES), and \(1.14\) (standard vs MART). Within the robust models, PGD versus TRADES remains clearly positive (\(0.28\)), TRADES versus MART is smaller but still reliable (\(0.22\)), and PGD versus MART is near zero (\(-0.01\)). At \(K=32\), contrast-derived probes outperform all three controls in the TRADES-versus-MART comparison (all \(p<10^{-3}\)), whereas no such advantage is present for PGD versus MART. 

\begin{figure}[t]
    \centering
    \includegraphics[width=\linewidth]{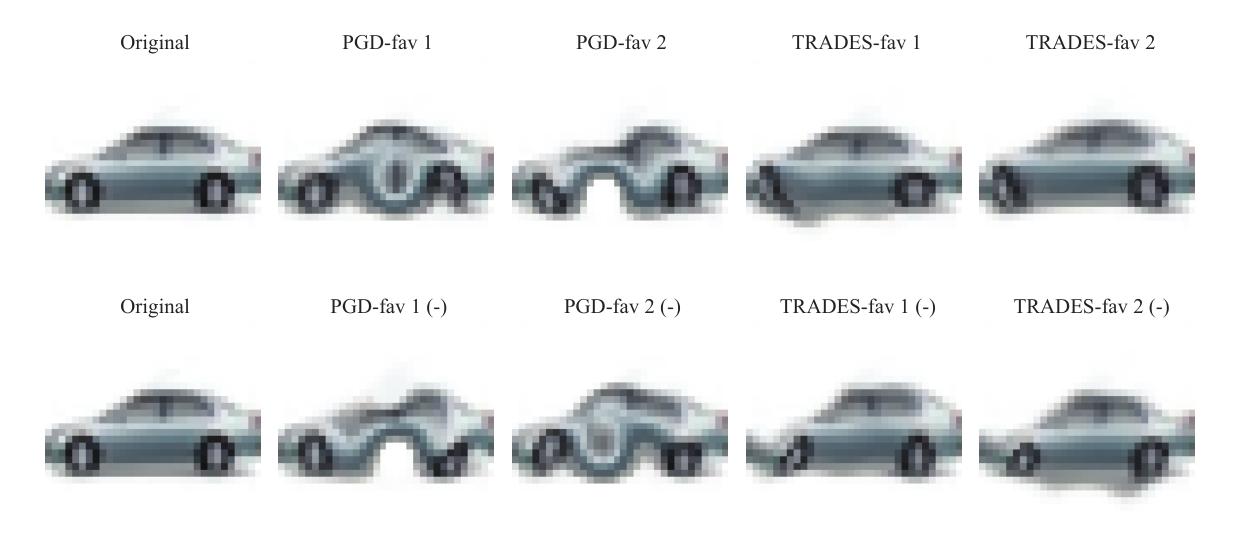}
    \caption{
    \textbf{Within-robust qualitative example.}
    Example probes for PGD versus TRADES at the main setting. Even though the quantitative effect is smaller than in the standard-versus-robust comparisons, the two probe sets still differ in visible character.
    }
    \label{fig:adv_std_qual_appendix}
\end{figure}

\begin{figure}[t]
    \centering
    \includegraphics[width=\linewidth]{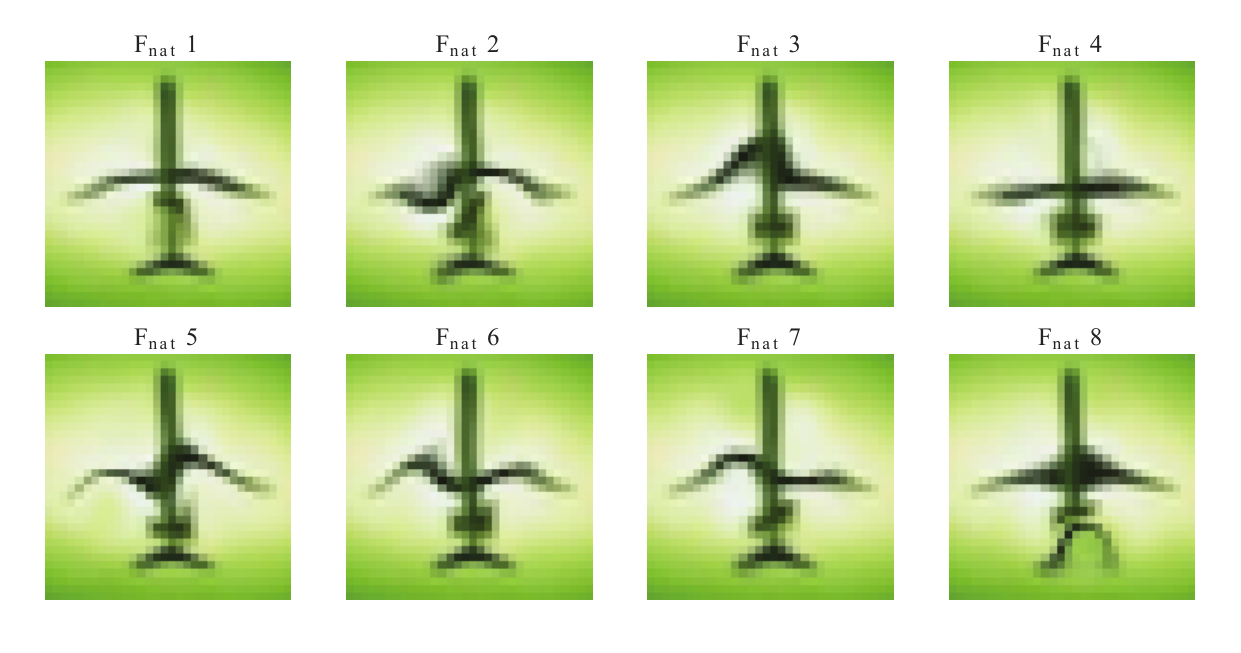}
    \caption{
    \textbf{Examples from the learned family basis.}
    Representative family directions used to instantiate finite probes. These basis elements are smooth local deformations rather than noisy pixel patterns, which helps interpret the learned probe sets as structured perturbations within a coherent family.
    }
    \label{fig:adv_std_basis_appendix}
\end{figure}

\paragraph{Learned family basis.}
In the robust-training experiment, the perturbation family is represented by a shared learned basis
\[
F_{\mathrm{nat}}=\{F_{\mathrm{nat},1},\dots,F_{\mathrm{nat},K_{\max}}\},
\]
where each \(F_{\mathrm{nat},j}\) is a learned smooth deformation field in pixel space. These basis elements are obtained by first constructing the \(72\)-flow base bank, estimating image-conditioned tangents in that bank, computing the tangent covariance on the family-learning split, and taking the top eigenvectors of that covariance as a learned maximal family basis. For family dimension \(K\), perturbations are then restricted to linear combinations of the first \(K\) learned basis elements. Thus \(F_{\mathrm{nat}}\) denotes the ambient learned family basis itself, not a set of contrast-derived probes. Figure~\ref{fig:adv_std_basis_appendix} shows representative basis elements themselves, whereas Figure~\ref{fig:adv_std_qual_appendix} shows selected probe directions obtained from between-group contrasts within that family.

Figures~\ref{fig:adv_std_qual_appendix} and~\ref{fig:adv_std_basis_appendix} provide qualitative support for this interpretation. The within-robust PGD-versus-TRADES example shows that the learned probes remain visibly different even when the quantitative effect size is smaller, while the family-basis examples show that the ambient perturbation family itself is smooth and structured.

\begin{figure*}[t]
    \centering
    \includegraphics[width=\textwidth]{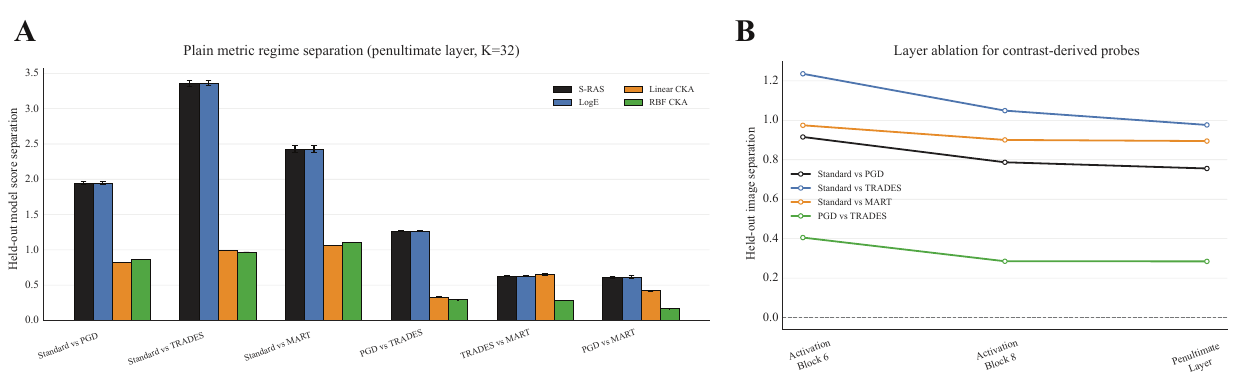}
    \caption{
    \textbf{Supporting analyses: direct model-space separation and layer ablation.}
    (A), direct model-space score separation at the pre-logit layer for \sras, Log-Euclidean distance, and activation-based baselines.
    (B), layer ablation for contrast-derived probes across block 6, block 8, and pre-logits.
    Direct model-space separation is already strong for several metrics, whereas the held-out probe effect is not confined to the final layer and is often already strong in intermediate blocks.
    }
    \label{fig:adv_std_plain_layer_appendix}
\end{figure*}

\paragraph{Direct model-space separation.}
The direct model-space analysis in Figure~\ref{fig:adv_std_plain_layer_appendix} uses the same discovery/test splits as the probe analysis, but does not instantiate finite image perturbations. For a regime pair \(A,B\), split \(r\), held-out model \(m\), and distance matrix \(D\) computed on a fixed layer and family dimension, define
\[
d_A(m)=\frac{1}{|\mathcal D_A^{(r)}|}
\sum_{j\in\mathcal D_A^{(r)}}D(m,j),
\qquad
d_B(m)=\frac{1}{|\mathcal D_B^{(r)}|}
\sum_{j\in\mathcal D_B^{(r)}}D(m,j),
\]
where \(\mathcal D_A^{(r)}\) and \(\mathcal D_B^{(r)}\) are the discovery models from the two regimes. The model-space score is
\[
u_D(m)=d_B(m)-d_A(m),
\]
so larger values indicate that \(m\) is closer to the discovery models from regime \(A\). The plotted model-space score separation is
\[
\operatorname{Sep}_D
=
\frac{1}{|\mathcal H_A^{(r)}|}
\sum_{m\in\mathcal H_A^{(r)}}u_D(m)
-
\frac{1}{|\mathcal H_B^{(r)}|}
\sum_{m\in\mathcal H_B^{(r)}}u_D(m),
\]
averaged over the \(10\) held-out splits, where \(\mathcal H_A^{(r)}\) and \(\mathcal H_B^{(r)}\) are the held-out models. We also compute the corresponding held-out model AUC from the same scores. Distances \(D\) are computed using \sras, Log-Euclidean distance, linear CKA distance, or RBF CKA distance on the same saved representations.

Figure~\ref{fig:adv_std_plain_layer_appendix} should be read as supporting context rather than the primary probe claim. Direct model-space regime separation is already strong for several metrics, which is why we focus the main result on finite perturbations and held-out image transfer. The layer ablation shows that the standard-versus-robust and PGD-versus-TRADES effects are not confined to the final layer, whereas PGD versus MART remains weak across layers. In addition to the main regime-pair comparisons, the saved analyses include all six pairwise regime contrasts at \(K=32\), \(K\)-sweeps for the probe battery, layer ablations across block 6, block 8, and pre-logits, and direct model-space separation baselines using \sras, Log-Euclidean distance, and activation-based comparisons.

\subsection{Biological static-gratings implementation}
\label{app:biology_details}

This appendix describes how the biological analysis approximates the theory in Section~\ref{sec:theory}. The theory assumes a smooth parameterized stimulus family \(x=\psi(s)\) and a representation map \(h_f(s)=f(\psi(s))\). In the Allen static-gratings analysis, the role of the representation map is played by the experiment-level population response map
\[
s\mapsto \mu_e(s),
\qquad
s=(\theta,\rho,\phi).
\]

\paragraph{Recordings and trial responses.}
For each experiment \(e\), we use the static-gratings stimulus table and retain the condition coordinates orientation, spatial frequency, and phase. The Allen static-gratings protocol contains \(6\) orientations, \(5\) spatial frequencies, and \(4\) phases, yielding \(120\) stimulus conditions \cite{AllenStaticGratings}. For each trial we compute response-minus-baseline activity,
\[
\Delta r_{t}=r^{\mathrm{response}}_{t}-r^{\mathrm{baseline}}_{t},
\]
using the precomputed trial response and baseline summaries. The condition mean
\[
\mu_e(s)=\E[\Delta r_t\mid s_t=s]
\]
is estimated by averaging trials assigned to the same grating condition.

\paragraph{Matched-count populations.}
Because Fisher scale depends strongly on population size, each benchmark uses matched-count cell subsampling. For the area benchmark we use \(n_{\mathrm{match}}=60\) cells per experiment; for the VISp-depth benchmark we use \(n_{\mathrm{match}}=45\). For each experiment, we draw \(100\) matched cell subsamples without replacement. For each subsample, condition means and pooled covariances are recomputed from the selected cells, and the reported experiment-level operator is the average over the \(100\) matched-subsample operators. Matched split-half reliability is computed by splitting trials within each condition, recomputing the full Fisher operator on the two halves, comparing the two half-operators with S-RAS, and averaging over \(2\) split-half repetitions per matched subsample. Full-population reliability diagnostics use \(8\) split-half repetitions.
\paragraph{Noise covariance.}
Let \(\Delta r_t\in\mathbb R^{n_{\mathrm{match}}}\) denote the response-minus-baseline vector for trial \(t\), and let
\[
\epsilon_t=\Delta r_t-\mu_e(s_t)
\]
be the residual after subtracting the condition mean. We concatenate residuals across static-grating conditions and estimate \(\widehat\Sigma_e\) with Ledoit--Wolf shrinkage covariance, using the centered residuals. The shrinkage coefficient is estimated separately for each experiment or matched subsample. Before inversion, we use the same SPD numerical constant as in the S-RAS lift, \(\varepsilon_{\mathrm{SPD}}=10^{-6}\): the covariance is replaced by \(\widehat\Sigma_e+\varepsilon_{\mathrm{SPD}}I\), symmetrized, and eigenvalues are floored at \(\varepsilon_{\mathrm{SPD}}\) before inversion. The noise-aware Fisher operator uses this regularized inverse covariance, whereas the naive baseline replaces the inverse covariance by the identity.

\paragraph{Finite-difference Jacobian.}
We estimate the Jacobian
\[
D\mu_e(s)=
\begin{bmatrix}
\partial_\theta\mu_e(s) &
\partial_\rho\mu_e(s) &
\partial_\phi\mu_e(s)
\end{bmatrix}
\]
by finite differences over the static-gratings grid. Thus this biological instantiation has coordinate dimension \(q=3\), response dimension \(n_{\mathrm{match}}\), \(D\mu_e(s)\in\mathbb R^{n_{\mathrm{match}}\times 3}\), and \(\widehat\Sigma_e\in\mathbb R^{n_{\mathrm{match}}\times n_{\mathrm{match}}}\). Orientation and phase are treated as circular coordinates, while spatial frequency is treated on the log scale
\[
\rho=\log_2(\text{spatial frequency}).
\]
The implementation uses orientation in radians, \(\rho\) in octaves, and the raw Allen phase coordinate for \(\phi\). No additional z-scoring or unit standardization is applied before computing finite differences. Specifically, orientation and phase use centered circular differences, while \(\rho\) uses centered differences across neighboring log-spatial-frequency values; spatial-frequency endpoints without both neighbors are omitted from the average.

\paragraph{Experiment-level operators.}
At each valid grid point \(s\), we form the local noise-aware metric
\[
I_e(s)=D\mu_e(s)^\top \widehat\Sigma_e^{-1}D\mu_e(s)
\]
and the unwhitened local metric
\[
M_e(s)=D\mu_e(s)^\top D\mu_e(s).
\]
The experiment-level summaries are
\[
F_e=\E_s[I_e(s)],
\qquad
G_e=\E_s[M_e(s)].
\]
These are the biological analogues of the projected expected Fisher operator and projected expected pullback metric in the main theory.
\paragraph{Shape-only normalization.}
For shape-only comparisons, we use trace normalization:
\[
\widehat F=\frac{F}{\Tr(F)},
\qquad
\widehat G=\frac{G}{\Tr(G)}.
\]
This differs from determinant or geometric-mean normalization. One-dimensional shape-only families are not reported, because every nonzero \(1\times1\) operator becomes identical after trace normalization.

\paragraph{Family-restricted comparisons.}
For each coordinate family \(Q\subseteq\{\theta,\rho,\phi\}\), let \(P_Q\) be the coordinate-selection matrix. We compute
\[
F_{e,Q}=P_Q^\top F_eP_Q,
\qquad
G_{e,Q}=P_Q^\top G_eP_Q.
\]

The families are
\[
\theta,\quad
\rho,\quad
\phi,\quad
\theta+\rho,\quad
\theta+\phi,\quad
\rho+\phi,\quad
\theta+\rho+\phi.
\]

\paragraph{Similarity and retrieval.}
All Fisher and naive operators are compared using the same regularized S-RAS score as in the model experiments. Each query experiment ranks donor-distinct candidate experiments by similarity. Top-1 accuracy records whether the nearest candidate has the same area or VISp-depth label. Diagonal AUC treats same-label pairs as positives and different-label pairs as negatives and evaluates threshold-free label separation. All method summaries are computed from the same donor-distinct query-record table, so Fisher, shape-only Fisher, naive geometry, CKA, RSA, decoder-profile, and mapping baselines are evaluated on identical query/candidate records.

\paragraph{Activation, decoder-profile, and mapping baselines.}
Let \(X_e\in\mathbb R^{120\times n_e}\) denote the condition-mean response matrix for experiment \(e\), with rows corresponding to the \(6\times5\times4\) static-grating conditions and columns to matched cells. The activation baseline is linear CKA,
\[
\operatorname{CKA}(X_e,X_{e'})
=
\frac{\|X_e^{c\top}X_{e'}^c\|_F^2}
{\|X_e^{c\top}X_e^c\|_F\,\|X_{e'}^{c\top}X_{e'}^c\|_F},
\]
where \(X^c\) denotes column-centered responses. The RSA baseline computes Euclidean representational dissimilarity matrices over the 120 conditions and reports the Spearman correlation between the upper-triangular entries of the two dissimilarity matrices.

The decoder-profile baseline summarizes each experiment by how well its condition-mean responses linearly decode the three stimulus coordinates. We flatten the static-gratings grid to \(120\) condition examples and use cross-validated multinomial logistic regression with standardized features to predict orientation, log-spatial-frequency index, and phase index separately. This gives a three-dimensional profile
\[
d_e=
\bigl(
\operatorname{Acc}_e(\theta),
\operatorname{Acc}_e(\rho),
\operatorname{Acc}_e(\phi)
\bigr).
\]
Decoder-profile similarity between experiments is the Pearson correlation of the finite entries of \(d_e\) and \(d_{e'}\). This is an auxiliary summary of which stimulus variables are linearly accessible from condition means; it is not a local-geometry or noise-aware Fisher summary.

The inter-experiment mapping baselines ask how well responses in one experiment predict responses in another across matched static-grating conditions. For each donor-distinct pair \((e,e')\), we use the same \(120\) condition rows, randomly split conditions into train/test halves for \(10\) repetitions, fit a map from \(X_e\) to \(X_{e'}\) on the training half, and score the mean test-set correlation across target cells. We compute this in both directions and average the two directional scores. The ridge mapping baseline uses dual ridge regression with \(\alpha\in\{0.1,1,10,100\}\) selected by generalized cross-validation within each split. The PLS mapping baseline uses PLS regression with
\[
n_{\mathrm{comp}}=\min(10,n_e,n_{e'},n_{\mathrm{train}}-1).
\]
These mapping baselines test whether the Fisher summaries add information beyond direct linear inter-experiment predictability.

\paragraph{Auxiliary biological diagnostics.}
Table~\ref{tab:biology_diag_auc} reports diagonal-AUC diagnostics for the same donor-distinct query records as the main top-1 table. Table~\ref{tab:biology_incremental} reports an auxiliary area-cohort pair-classification analysis testing whether Fisher summaries add information beyond activation, decoder-profile, and mapping features.

For the pair-classification analysis, each donor-distinct experiment pair is represented by a feature vector containing the pairwise scores from the specified feature set. The binary label is whether the pair shares the same area label in the area cohort. We train a logistic-regression classifier with standardized features and evaluate cross-validated AUC using group splits over unordered experiment-pair groups, so that the same experiment pair does not contribute to both training and test folds. The baseline feature set contains activation CKA, RSA, decoder-profile correlation, ridge mapping score, and PLS mapping score. The ``baseline + Fisher'' rows add one additional Fisher or naive-geometry similarity score to this baseline. These diagnostics are supportive rather than the primary biological benchmark; the main text reports donor-distinct top-1 retrieval because it gives the most direct identity-matching interpretation.

\begin{table}[t]
\centering
\caption{\textbf{Biological diagonal-AUC diagnostics.} Diagonal AUC is a threshold-free same-label versus different-label diagnostic computed from the same donor-distinct query records as the top-1 retrieval table. Top-1 accuracy is the main table metric; diagonal AUC is included to show softer label separation.}
\label{tab:biology_diag_auc}
\begin{tabular}{lrrrr}
\toprule
Cohort & Full Fisher & CKA & Shape-only Fisher & Naive full \\
\midrule
Area & 0.556 & 0.567 & 0.604 & 0.518 \\
VISp depth & 0.549 & 0.547 & 0.540 & 0.513 \\
\bottomrule
\end{tabular}
\end{table}

\begin{table}[t]
\centering
\caption{\textbf{Complementarity of biological Fisher summaries.} Auxiliary area-cohort pair-classification analysis using cross-validated AUC. The baseline uses activation, decoder-profile, and inter-experiment mapping features. Adding trace-normalized shape-only Fisher gives the largest gain, suggesting that allocation contains information not fully captured by standard activation and mapping baselines. }
\label{tab:biology_incremental}
\begin{tabular}{lr}
\toprule
Feature set & Area CV AUC \\
\midrule
Baseline activation / decoder / mapping features & 0.589 \\
Baseline + full Fisher & 0.600 \\
Baseline + shape-only Fisher & 0.646 \\
Baseline + naive full & 0.588 \\
Baseline + naive shape-only & 0.632 \\
\bottomrule
\end{tabular}
\end{table}

\FloatBarrier
\clearpage

\end{document}